\documentclass[author-year, review]{elsarticle}
\usepackage[margin=3cm]{geometry}

\usepackage{lineno,hyperref}
\modulolinenumbers[1]

\journal{Journal of \LaTeX\ Templates}

\usepackage{setspace}
\usepackage{amssymb}
\usepackage{amsmath}
\usepackage{latexsym}
\usepackage{mathabx}
\usepackage{amsthm}
\usepackage{epigraph}
\usepackage{booktabs}
\usepackage{braket}
\usepackage{bbm}
\usepackage{appendix}
\usepackage{wrapfig}
\usepackage{graphicx}
\usepackage{array,multirow}
\usepackage{float}
\usepackage{pdflscape}
\usepackage{epstopdf}
\usepackage{pdfpages}
\usepackage{pgfplots}
\usepackage{longtable}
\usepackage{changepage}
\usepackage{soul}
\soulregister\cite7
\soulregister\ref7
\soulregister\eqref7
\soulregister\pageref7
\usepackage{cancel}
\usetikzlibrary{matrix,arrows,decorations.pathmorphing,backgrounds,fit,positioning}
\DeclareMathOperator*{\argmax}{arg\,max}
\DeclareMathOperator*{\argmin}{arg\,min}
\DeclareMathOperator{\sign}{sign}
\newcommand{\norm}[1]{\left\lVert#1\right\rVert}

\newcommand{\abs}[1]{\left\lvert#1\right\rvert}
\newcommand\RotText[1]{\fontsize{9}{9}\selectfont
  \rotatebox[origin=c]{90}{\parbox{3.5cm}{%
\centering#1}}}
\usepackage{pifont}
\newcommand{\cmark}{\ding{51}}%
\theoremstyle{definition}
\newtheorem{theorem}{Theorem}

\newtheorem{corollary}{Corollary}

\usepackage[font={footnotesize}]{caption}
\usepackage[font={scriptsize}]{subcaption}

\usepackage{algorithm,algorithmic}
\makeatletter
\renewcommand{\ALG@name}{Pseudocode}
\makeatother

\makeatletter
\def\ps@pprintTitle{%
 \let\@oddhead\@empty
 \let\@evenhead\@empty
 \def\@oddfoot{}%
 \let\@evenfoot\@oddfoot}
\makeatother

\makeatletter
\newcommand{\mathleft}{\@fleqntrue\@mathmargin0pt}
\newcommand{\mathcenter}{\@fleqnfalse}
\makeatother

\begin{document}
\nolinenumbers

\begin{frontmatter}

\title{A Robust Twin Parametric Margin Support Vector Machine for Multiclass Classification}

\author[DeLeone_address]{Renato De Leone}

\author[Maggioni_address]{Francesca Maggioni\corref{mycorrespondingauthor}}
\cortext[mycorrespondingauthor]{Corresponding author}
\ead{francesca.maggioni@unibg.it}

\author[Maggioni_address]{Andrea Spinelli}

\address[DeLeone_address]{School of Science and Technology, University of Camerino, Via Madonna delle Carceri 9, Camerino 62032, Italy}
\address[Maggioni_address]{Department of Management, Information and Production Engineering, University of Bergamo, Viale G. Marconi 5, Dalmine 24044, Italy}

\begin{abstract}
In this paper, we introduce novel Twin Parametric Margin Support Vector Machine (TPMSVM) models designed to address multiclass classification tasks under feature uncertainty. To handle data perturbations, we construct bounded-by-norm uncertainty set around each training observation and derive the robust counterparts of the deterministic models using robust optimization techniques. To capture complex data structure, we explore both linear and kernel-induced classifiers, providing computationally tractable reformulations of the resulting robust models. Additionally, we propose two alternatives for the final decision function, enhancing models' flexibility. Finally, we validate the effectiveness of the proposed robust multiclass TPMSVM methodology on real-world datasets, showing the good performance of the approach in the presence of uncertainty.
\end{abstract}

\begin{keyword}
Machine Learning \sep Support Vector Machine \sep Robust Optimization \sep Multiclass Classification
\end{keyword}

\end{frontmatter}

\nolinenumbers

\section{Introduction}
Supervised classification is one of the most extensively studied tasks in \emph{Machine Learning} (ML) thanks to its wide variety of application fields (see \cite{DiaAma2023}). Although deep learning algorithms and neural networks currently represent the state-of-the-art paradigms in supervised classification, such methodologies do not always guarantee strong predictive accuracies, particularly when applied to tabular data (see \cite{GunSepBaeOskLem2021}). Additionally, these techniques tend to be inefficient for low-dimensional inputs because of their high degree of overparametrization (see \cite{CipGon2022}). For this reason, the design of innovative data-driven approaches in addressing supervised classification tasks remains a significant field of research (see \cite{MalLopCar2022}).

According to \cite{Pedro2017}, two optimization-based approaches to supervised classification can be identified in the literature. The first relies on classification rules derived from mathematical programming models. These techniques aim to minimize misclassification error by optimizing observable criteria. This approach, primarily developed within the Operations Research literature, can be seen as an extension of the Fisher's \textit{Linear Discriminant Function} (see \cite{fisher1936use}). The second group of methods, rooted in statistical learning theory, focuses on developing predictive algorithms with guaranteed generalization ability. These classifiers are constructed by optimizing sample accuracy measures associated with distribution-free bounds on misclassification probabilities. This line of research has led to the development of \textit{Support Vector Machines} (SVMs, see \cite{VapChe1974}).

Classical SVMs consist in identifying the best classifier in the form of a hyperplane or a kernel-induced decision boundary that geometrically separates two sets of labelled training data using a structural risk minimization principle (see \cite{Vap1995}). Due to their simplicity and efficiency, SVMs have received strong attention in the ML literature (see \cite{JiaSid2020}). Significant methodological developments have been devised (see, for instance, \cite{CarNogRom2017, CorVap1995, JayKheCha2007, LiuPot2009,MarDeL2020, SchSmoWilBar2000}), making SVMs one of the most powerful tools for supervised classification (see \cite{Pedro2017}). The range of applications includes finance (see \cite{DeL2010,IslKwoKim2025}), scheduling (see \cite{GolZhaDre2023}), and business analytics (see \cite{MalDomOlaVerb2021,MalLopVai2020}), to name a few.

Among all the possible SVM variants, in this paper, we focus on the \emph{Twin Parametric Margin Support Vector Machine} (TPMSVM, see \cite{Peng2011}). According to this method, two nonparallel classifiers, one for each class, are identified such that the training observations of the other class are as far as possible from the opposite classifier. The main advantage of this approach lies in its reduced computational complexity since each of the two classifiers is the solution of a small-sized optimization model. Empirical evidence suggests that TPMSVM achieves superior predictive accuracy compared to other SVM-type techniques (see \cite{Peng2011,DeLMagSpiLOD2024}).

Optimization approaches for supervised classification, including SVMs, were traditionally designed to tackle binary classification tasks. However, many real-world applications involve multigroup problems (see \cite{MaqDam2023}), requiring the development of specific methods to be addressed (see \cite{WesWat1998}). Typically, these problems are decomposed into a finite sequence of binary classification tasks, whose solutions are finally combined into an aggregate decision function (see \cite{HsuLin2002}). Depending on how the decomposition and the following reconstruction are performed, various approaches have been proposed (see \cite{DingZhaZhaZhaXue2019}). Nevertheless, multiclass classification problems remain less explored in the ML literature due to their higher computational complexity (see \cite{PengKouWangShi2011}). For this reason, developing new algorithms for multiclass SVMs is considered a promising research area (see \cite{Pedro2017}).

In supervised classification methods, input data are supposed to be known exactly when training the models. However, real-world observations are often subject to noise and perturbations due to errors in the data collection process or to limited precision of the measurement instruments. In recent years, various techniques have been explored to address uncertainty in classification problems. Among these, \emph{Robust Optimization} (RO) is one of the most widely studied paradigms in the ML literature (see \cite{Ben-TalElGNem2009}). RO techniques protect the optimization model against the worst possible realizations of the random parameters within a prescribed uncertainty set. Different uncertainty sets lead to distinct solutions, depending on the level of conservatism and on the severity of the perturbation. Common choices to define the uncertainty sets often employ $\ell_p$-norms (see \cite{BerDunPawZhu2019, TraGil2006}). When RO techniques are applied, the predictive performance of supervised classification algorithms is enhanced, especially in the case of SVMs (see \cite{FacMagPot2022,MagSpi2023}). Therefore, designing new robust models for SVMs represents a relevant research direction.

In this paper, we present novel TPMSVM-type models aiming at separating multiple classes of data under feature uncertainty. The formulation builds upon the work of \cite{Peng2011} and introduce two key innovations. First, we handle multigroup classification problems instead of two-group tasks. Second, we apply robust optimization techniques to protect the multiclass models against uncertainty. Specifically, we consider bounded-by-$\ell_p$-norm uncertainty sets around each training observation and derive the robust counterpart of the deterministic approach. In addition, we provide computationally tractable reformulations of the robust models in the form of \emph{Second Order Cone Programming} (SOCP) models. To improve the generalization capability of the proposal, all the results are derived using both linear and kernel-induced classifiers. To the best of our knowledge, this is the first work in the ML literature to introduce a robust TPMSVM approach for addressing multiclass classification tasks using linear and kernel-induced classifiers under bounded-by-$\ell_p$-norm uncertainty sets.

The main contributions of this paper are four-fold and can be summarized as follows:
\begin{itemize}
\item to propose new multiclass TPMSVM-type models with both linear and kernel-induced classifiers;
\item to formulate the robust counterparts of the deterministic multiclass models using bounded-by-$\ell_p$-norm uncertainty sets;
\item to derive computationally tractable reformulations of the robust multiclass models as SOCP models for typical choices of the $\ell_p$-norm;
\item to provide numerical experiments on real-world datasets with the aim of evaluating the performance of the proposed approach and showing the advantages of explicitly accounting for uncertainty in the novel TPMSVM formulations.
\end{itemize}

The remainder of the paper is organized as follows. Section \ref{sec_literature_review} reviews the existing literature on the problem. Section \ref{sec_background} introduces basic facts on binary TPMSVM model. In Section \ref{sec_multiclass_deterministic} the deterministic multiclass models are designed, while in Section \ref{sec_multiclass_robust} the robust counterparts are presented. Section \ref{sec_numerical_results} reports computational results to evaluate the accuracy of the proposed formulations. Finally, in Section \ref{sec_conclusions} conclusions and future works are discussed.

Throughout the paper, all vectors are column vectors, unless transposed to row vectors by the superscript ``$^\top$''. For $p\in[1,\infty]$ and $a\in\mathbb{R}^n$, $\norm{a}_p$ is the $\ell_p$-norm of $a$. The dot product in a inner product space $\mathcal{H}$ will be denoted by $\langle \cdot,\cdot\rangle_{\mathcal{H}}$. If $\mathcal{H}=\mathbb{R}^n$ and $a,b\in\mathbb{R}^n$, $a^\top b$ is equivalent to $\langle a, b \rangle_{\mathcal{H}}$. If $\mathcal{A}$ is a set, $\abs{\mathcal{A}}$ represents its cardinality. In addition, we denote by $e_n$ the column vector of ones in $\mathbb{R}^n$.

\section{Literature review} \label{sec_literature_review}
Classical SVM is first introduced in \cite{VapChe1974} with the aim of solving classification problems involving two linearly separable sets. The generalization of the linear approach is proposed in \cite{BosGuyVap1992} by including nonlinear decision boundaries induced by kernel functions. Cases of non perfectly separable training data are considered in \cite{CorVap1995}, where a vector of slack variables is introduced into the SVM model. The resulting formulation looks for a trade-off between the maximization of the margin, as in the classical SVM approach of \cite{VapChe1974}, and the minimization of an empirical risk related to the slack variables.

When the classification problem involves more than two groups, solution strategies in the SVM literature fall into two categories: \emph{all-together} methods (see \cite{CraSin2001}) and \emph{decomposition-reconstruction} methods (see \cite{HsuLin2002}). The former considers all training data at the same time within a single optimization model to derive a unified classifier (see \cite{WesWat1998,BreBen1999,Yaj2005}). In contrast, the latter decomposes the problem into a sequence of classification tasks, solving each independently and then combining the resulting SVM-classifiers into an aggregate multiclass decision function. Within this paradigm, various formulations have been proposed. In the \emph{One-Versus-All} strategy (OVA, see \cite{RifKla2004}), a classifier is constructed for each class aiming to separate data points belonging to that class from all others. In the \emph{One-Versus-One} approach (OVO, see \cite{LiTanYan2002}), a binary classification problem is solved for each pair of classes. Conversely, in the \emph{One-Versus-One-Versus-Rest} strategy (OVOVR, see \cite{AngParCat2003}), each subproblem focuses on the separation of a pair of classes together with all the remaining samples by means of two classifiers. These classifiers are close to their respective class, while as far as possible from the other. At the same time, all remaining points are restricted in a region between the two classifiers. Decomposition-reconstruction methods are generally considered as the most effective for multiclass classification problems (see \cite{DuZhaCheSunCheShao2021}), especially due to the high computational complexity of all-together methods when handling large datasets (see \cite{DingZhaZhaZhaXue2019}). However, efforts have been made in the literature to overcome these limitations and to unify existing all-together methods (see \cite{DogGlaIge2016}).

To avoid low classification accuracy when training data are affected by perturbations, optimization under uncertainty techniques are employed within the SVM context. Specifically, \textit{Robust Optimization} (RO, see \cite{XuCarMan2009}), \textit{Distributionally Robust Optimization} (DRO, see \cite{RahMeh2022}) and \textit{Chance-Constrained Programming} (CCP, see \cite{ShaDenRus2009}) are some of the most extensively studied approaches. Robust formulations of standard classification methods, including SVM, are provided in \cite{BerDunPawZhu2019}. Methodological advancements and applications of RO techniques to SVMs are discussed in \cite{FacMagPot2022,MagSpi2023,MagFacGheManBonORAHS2022,MagSpiODS2024,PiaSpiMagBedMes2025}. Recently, an adjustable RO approach for SVM under uncertainty has been proposed in \cite{HooSeiMir2025}. Within the multiclass framework, in \cite{ZhongFuk2007} a RO model for SVM is derived through piecewise-linear functions, robustifying the approach of \cite{BreBen1999} in the case of ellipsoidal uncertainty sets. Finally, \cite{Kha-ShiBab-AzaHos-NodPar2023} and \cite{LinFangFangGao2024,LinFangFangGaoLuo2024} investigate the integration of CCP and DRO techniques into linear and kernel-induced SVM models, respectively.

Up to this point, we have outlined the general framework of SVMs, including approaches for multiclass classification problems and focusing on optimization techniques for uncertain data. In the following, we discuss SVM variants that are related to the current proposal. We start with the $\nu$\emph{-Support Vector Classification} ($\nu$-SVC) approach, designed in \cite{SchSmoWilBar2000}. Compared with the classical SVM presented in \cite{VapChe1974}, a positive parameter $\nu$ is introduced in the risk minimization function to bound the fractions of support vectors and misclassification errors. Relying on this approach, in \cite{Hao2010} a \emph{parametric margin} model (the par-$\nu$-SVM) is formulated to deal with cases of heteroscedastic noise. Rather than dealing with parallel hyperplanes as in \cite{CorVap1995}, the \emph{TWin Support Vector Machine} (TWSVM) considers a pair of nonparallel classifiers as solutions of two small-sized SVM-type models (see \cite{JayKheCha2007}). Consequently, the computational complexity of TWSVM is much reduced compared with the SVM approach of \cite{CorVap1995}. Due to its favourable performance, especially when handling large datasets, many variants of the TWSVM have been devised in the ML literature: \emph{Least Squares TWSVM} (LS-TWSVM, see \cite{KumGop2009}), \emph{Projection TWSVM} (P-TWSVM, see \cite{CheYanYeLian2011}), \emph{Twin Parametric Margin SVM} (TPMSVM, see \cite{Peng2011}), \emph{Pinball loss TWSVM} (Pin-TWSVM, see \cite{XuYangXianli2017}), \emph{New Fuzzy TWSVM} (NFTWSVM, see \cite{ChenWu2018}). For a comprehensive overview on recent developments on TWSVM the reader is referred to \cite{TanRajRasShaoGan2022}.

Among all the possible TWSVM alternatives, in this paper we focus on the TPMSVM. This formulation combines the contributions of the TWSVM and the par-$\nu$-SVM. Specifically, the TPMSVM constructs two nonparallel classifiers, each of them determining the positive or negative parametric margin, by solving two small-sized optimization models. Therefore, this approach integrates the fast learning speed of the TWSVM and the flexible parametric margin of the par-$\nu$-SVM. Alternative TPMSVM-based formulations are \emph{Structural TPMSVM} (STPMSVM, see \cite{PengWangXu2013}), \emph{Least Squares TPMSVM} (LSTPMSVM, see \cite{ShaoWangChenDeng2013}), \emph{Smooth TPMSVM} (STPMSVM, see \cite{WangShaoWu2013}), \emph{Centroid-based TPMSVM} (CTPSVM, see \cite{PengKongChen2015}), \emph{Truncated} \emph{Pinball} \emph{Loss} \emph{TPMSVM} (TPin-TSVM, see \cite{WangXuZhou2021}).

Similar to SVM, TWSVM and its variants were originally designed for binary classification. To address multiclass classification problems, both all-together and decomposition-reconstruction methods have been explored in the TWSVM literature (see \cite{DingZhaZhaZhaXue2019,WangXuZhou2021,XieHoneXieGaoShiLiu2013,XuGuoWang2013,DuZhaCheSunCheShao2021}). In addition to these strategies, we mention the \textit{Directed Acyclic Graph TWSVM} (DAG TWSVM, \cite{TomAga2015}), the \textit{Binary Tree TWSVM} (BT TWSVM, \cite{ShaoChenHuanYanDen2013}), and the \textit{Multiple Birth SVM} (MBSVM, \cite{YangShaoZhang2013}). A comprehensive survey on multiclass formulations specifically designed for TWSVM can be found in \cite{DingZhaZhaZhaXue2019}.

In the context of optimization under uncertainty approaches for TWSVM, various techniques have been investigated. Within the RO framework, in \cite{PengXu2013} a \emph{Robust Minimum Class Variance} model (RMCV-TWSVM) is introduced. To handle data uncertainty, a pair of class variance-covariance matrices is considered, with uncertainty sets defined according to the Frobenius norm. In \cite{QiTianShi2013}, two nonparallel classifiers are proposed in the case of ellipsoidal uncertainty sets. The corresponding model, called R-TWSVM, is then reformulated as a SOCP model. Instead of convex hulls to represent the training patterns, the \textit{Robust NonParallel SVM} (RNPSVM, see \cite{LopMalCar2019}) and the TWin SOCP-SVM (see \cite{MalLopCar2016}) consider ellipsoids defined by the first two moments of the class distributions. The models are formulated through CCP techniques, leading to robust SOCP models. A similar CCP approach is employed in \cite{LopMalCar2017} for the case of twin multiclass SVM (TWin-KSOCP). Recently, in \cite{SahSal2022} an improved CCP version of RNPSVM, named IRNPSVM, has been designed to control the number of missing data. Regarding techniques for multiclass classification problems, \cite{DeLMagSpiLOD2024} presents a robust TPMSVM model with an application to vehicle emissions. To the best of our knowledge, this is the first contribution to handle multi-group classification problems using a robust TPMSVM-based methodology. However, \cite{DeLMagSpiLOD2024} considers only linear classifiers and spherical uncertainty sets, without addressing cases with general kernel functions or bounded-by-$\ell_p$-norm uncertainty sets.

All the approaches discussed so far on the TWSVM and its variants are schematically reported in Figure \ref{fig_TWSVMscheme} and listed in Table \ref{tab_horizontalSVM_literature_review}.

As discussed above, although robust optimization techniques have been applied to TWSVMs, there remains a gap in the literature concerning robust TPMSVM models. In this regard, the contributions of this paper differ from previous work in several aspects. First of all, we present novel TPMSVM-type models to address multiclass classification tasks, extending the binary approach of \cite{Peng2011}, both for linear and kernel-induced classifiers. Secondly, we consider bounded-by-$\ell_p$-norm uncertainty sets around training observations, employing general kernel functions for nonlinear classification problems. Thirdly, we derive the robust counterpart of the deterministic multiclass TPMSVM formulations, protecting the models against feature uncertainty. Finally, we provide tractable reformulations for all robust models as SOCP models, offering clear advantages in terms of computational efficiency.

\begin{figure}[h!]
\resizebox{\textwidth}{!}{
\begin{tikzpicture}[main/.style = {draw}] 

\node[main] (TPMSVM) {
\begin{tabular}{c}
\textbf{TPMSVM}\\
\footnotesize Peng \cite{Peng2011}
\end{tabular}
};

\node[main] (parnuSVM) [above left =12mm and -10mm of TPMSVM] {
\begin{tabular}{c}
\textbf{par-$\nu$-SVM}\\
\footnotesize Hao \cite{Hao2010}
\end{tabular}
};

\node[main] (TWSVM) [above right = 12mm and 2mm of TPMSVM] {
\begin{tabular}{c}
\textbf{TWSVM}\\
\footnotesize Jayadeva et al. \cite{JayKheCha2007}
\end{tabular}
};

\node[main] (varTPMSVM)  [below left = 12mm and -11mm of TPMSVM] {
\begin{tabular}{c}
\textbf{Variants of TPMSVM}
\end{tabular}
};

\node[text width = 0.1cm] (listvarTPMSVM) [below left = 3mm and 0mm of varTPMSVM] {
$
\begin{tabular}{c}
\small Deterministic\\
\small approaches
\end{tabular}
\left\{
 \begin{tabular}{l}
      $-$ STPMSVM \footnotesize Peng et al. \cite{PengWangXu2013}\\
     $-$ LSTPMSVM \footnotesize Shao et al. \cite{ShaoWangChenDeng2013}\\
      $-$ STPMSVM \footnotesize Wang et al. \cite{WangShaoWu2013}\\
           $-$ CTPSVM \footnotesize Peng et al. \cite{PengKongChen2015}\\
      $-$ OVOVR-TPMSVM \footnotesize Du et al. \cite{DuZhaCheSunCheShao2021}\\
      $-$ TPin-TSVM \footnotesize Wang et al. \cite{WangXuZhou2021}
    \end{tabular}
    \right.
    $
        \begin{tabular}{l}
    \text{}
    \end{tabular}
$
\begin{tabular}{c}
\small {Optimization}\\
\small {under}\\
\small {uncertainty}\\
\small {approaches}
\end{tabular}
\left\{
 \begin{tabular}{l}
      $-$ RO-TPMSVM \footnotesize De Leone et al. \cite{DeLMagSpiLOD2024}
    \end{tabular}
    \right.
    $
};

\node[main] (varTWSVM)  [below right = 12mm and -15mm of TWSVM] {
\begin{tabular}{c}
\textbf{Variants of TWSVM}
\end{tabular}
};

\node[text width = 0.1cm] (listvarTWSVM) [below left = 2mm and 0mm of varTWSVM] {
$
\left.
 \begin{tabular}{l}
      $-$ LS-TWSVM \footnotesize Arun Kumar \& Gopal \cite{KumGop2009} \text{}\\
      $-$ P-TWSVM \footnotesize Chen et al. \cite{CheYanYeLian2011}\\
      $-$ BT TWSVM \footnotesize Shao et al. \cite{ShaoChenHuanYanDen2013}\\
      $-$ OVA-TWSVM \footnotesize Xie et al., Wang et al. \cite{WangXuZhou2021,XieHoneXieGaoShiLiu2013}\\

      $-$ OVO-TWSVM \footnotesize Ding et al. \cite{DingZhaZhaZhaXue2019}\\
      $-$ OVOVR-TWSVM \footnotesize Xu et al. \cite{XuGuoWang2013}\\
      $-$ MBSVM \footnotesize Yang et al. \cite{YangShaoZhang2013}\\
       $-$ Pin-TWSVM \footnotesize Xu et al. \cite{XuYangXianli2017}\\
      $-$ NFTWSVM \footnotesize Chen \& Wu \cite{ChenWu2018}\\
      $-$ DAG TWSVM \footnotesize Tomar \& Agarwal \cite{TomAga2015}
    \end{tabular}
    \right\}
  \begin{tabular}{c}
\small Deterministic\\
\small approaches
  \end{tabular}
    $
    \begin{tabular}{l}
    \text{}
    \end{tabular}
$
  \left.
    \begin{tabular}{l}
      $-$ RMCV-TWSVM \footnotesize Peng \& Xu \cite{PengXu2013}\\
      $-$ R-TWSVM \footnotesize Qi et al. \cite{QiTianShi2013}\\
      $-$ TWin SOCP-SVM \footnotesize Maldonado et al. \cite{MalLopCar2016}\\
      $-$ TWin-KSOCP \footnotesize L\'opez et al. \cite{LopMalCar2017}\\
      $-$ RNPSVM \footnotesize L\'opez et al. \cite{LopMalCar2019}\\
      $-$ IRNPSVM \footnotesize Sahleh \& Salahi \cite{SahSal2022}
    \end{tabular}
  \right\}
  \begin{tabular}{c}
\small Optimization\\
\small under\\
\small uncertainty\\
\small approaches
  \end{tabular}
$
};

\draw[-latex, thick] (parnuSVM) -- (TPMSVM); 
\draw[-latex, thick] (TWSVM) -- (TPMSVM); 
\draw[-latex, thick] (TPMSVM) -- (varTPMSVM); 
\draw[-latex, thick] (TWSVM) -- (varTWSVM); 

\draw[semithick] (-5,3.5) -- (13,3.5);
\draw[semithick] (-5,-8) -- (13,-8);
\draw[semithick] (-5,3.5) -- (-5,-8);
\draw[semithick] (13,-8) -- (13,3.5);
\end{tikzpicture}
}
\caption{Scheme of the selected TWSVM literature review. The models are distinguished in deterministic and optimization under uncertainty approaches.
}
        \label{fig_TWSVMscheme}
\end{figure}
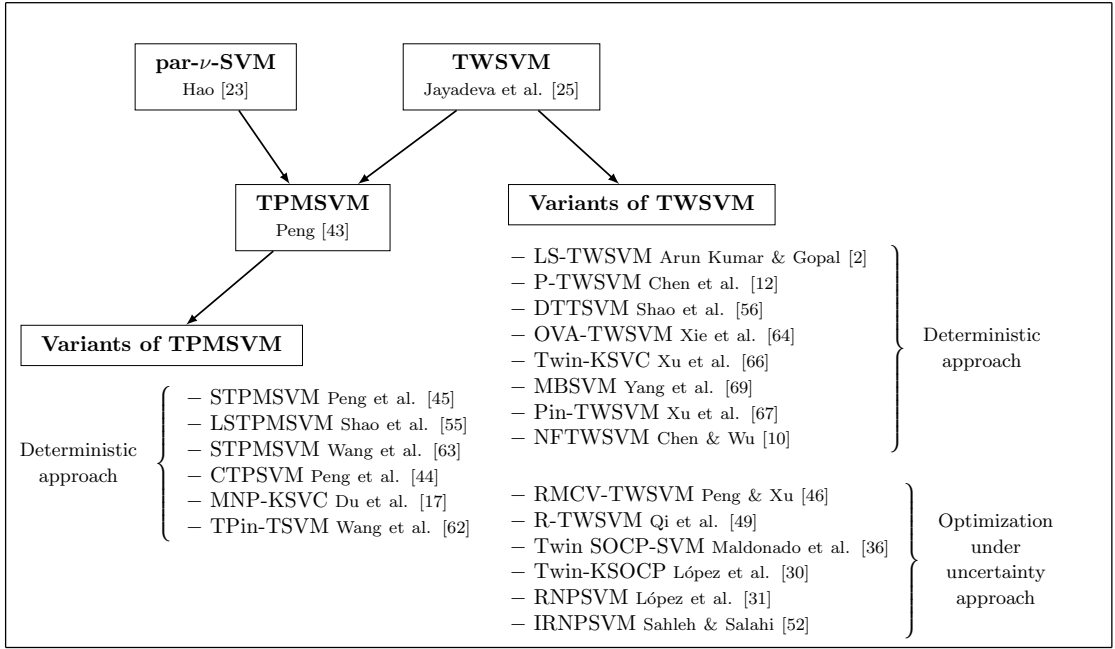

\begin{table}[h!]
\centering
\resizebox{\textwidth}{!}{
    \begin{tabular}{@{} c l *{23}c}
        & & \RotText{\small Jayadeva et al. \cite{JayKheCha2007}} & \RotText{\small Arun Kumar \& Gopal \cite{KumGop2009}} & \RotText{\small Chen et al. \cite{CheYanYeLian2011}} & \RotText{\small Peng \cite{Peng2011}} & \RotText{\small Peng and Xu \cite{PengXu2013}} & \RotText{\small Peng et al. \cite{PengWangXu2013}} & \RotText{\small Qi et al. \cite{QiTianShi2013}} & \RotText{\small Shao et al. \cite{ShaoWangChenDeng2013}} & \RotText{\small Shao et al. \cite{ShaoChenHuanYanDen2013}} & \RotText{\small Wang et al. \cite{WangShaoWu2013}} & \RotText{\small Xie et al. \cite{XieHoneXieGaoShiLiu2013}} & \RotText{\small Xu et al. \cite{XuGuoWang2013}} & \RotText{\small Yang et al. \cite{YangShaoZhang2013}} & \RotText{\small Peng et al. \cite{PengKongChen2015}} & \RotText{\small Maldonado et al. \cite{MalLopCar2016}} & \RotText{\small Maldonado et al. \cite{LopMalCar2017}} & \RotText{\small Xu et al. \cite{XuYangXianli2017}} & \RotText{\small Chen \& Wu \cite{ChenWu2018}} & \RotText{\small Lopez et al. \cite{LopMalCar2019}} & \RotText{\small Du et al. \cite{DuZhaCheSunCheShao2021}} & \RotText{\small Wang et al. \cite{WangXuZhou2021}} & \RotText{\small Sahleh \& Salahi \cite{SahSal2022}} &  \RotText{\small De Leone et al. \cite{DeLMagSpiLOD2024}}
        \\
       \toprule
\multirow{2}{*}{TWSVM} & Linear classifier & \cmark & \cmark & \cmark & \cmark & \cmark & \cmark & \cmark & \cmark & \cmark & \cmark & \cmark &  \cmark & \cmark &  \cmark &  \cmark & \cmark & \cmark & \cmark & \cmark & \cmark & \cmark & \cmark & \cmark
\\
\cmidrule{2-25}
& Nonlinear classifier & \cmark & \cmark & & \cmark & \cmark & \cmark & \cmark & \cmark & & \cmark & \cmark &  \cmark &  \cmark &  \cmark & \cmark & \cmark & \cmark & \cmark & \cmark & \cmark & \cmark & \cmark
\\
\cmidrule[0.55pt]{1-25}
\multirow{2}{*}{Classification} & Binary & \cmark & \cmark & \cmark & \cmark & \cmark & \cmark & \cmark & \cmark & & \cmark & & & &   \cmark & \cmark & & \cmark & \cmark & \cmark && \cmark & \cmark
\\
\cmidrule{2-25}
& Multiclass &&&&&&&&&  \cmark & &  \cmark &  \cmark &  \cmark &&&\cmark &&& \cmark & \cmark & & & \cmark
\\
\cmidrule[0.55pt]{1-25}
Optimization & Ellipsoidal RO & &  &  & & &   & \cmark & &&&&&&&&&&&&&&& \cmark
\\
\cmidrule{2-25}
under uncertainty & Matrix RO &  &  &  &   & \cmark & 
\\
\cmidrule{2-25}
approach & Chance-Constrained&&&&&&&&&&&&&&& \cmark &  \cmark &&& \cmark & && \cmark
\\
        \bottomrule
    \end{tabular}}
    \caption{A selected TWSVM literature review. In the first row of the table the contributions are listed in chronological order. Second and third rows specify the type of TWSVM (linear or nonlinear). Rows four and five consider binary versus multiclass classification. Finally, the optimization under uncertainty methodologies employed in the articles are explored in rows six to eight.} \label{tab_horizontalSVM_literature_review}
\end{table}

\section{Prior work} \label{sec_background}
In this section, we briefly recall the methods that are relevant for our proposal. Specifically, in Section \ref{sec_linearTPMSVM} we focus on the linear TPMSVM approach, while in Section \ref{sec_nonlinearTPMSVM} we discuss its extension to the cases of nonlinear kernel-induced decision boundaries. Both formulations are designed for addressing binary classification tasks and provide an initial framework for our novel multiclass approach.

\subsection{The binary TPMSVM for linear classification} \label{sec_linearTPMSVM}
Let $\{x^i,y_i\}_{i=1}^m$ be the set of training observations, where $x^i\in \mathbb{R}^n$ is the vector of features, and $y_i\in\{-1,+1\}$ is the label of the $i$-th data point, denoting the class to which it belongs. We assume that each of the two categories is composed by $m_{-}$ and $m_{+}$ observations, respectively, with $m_{-}+m_{+}=m$. We denote by $X_{-}\in\mathbb{R}^{n\times m_{-}}$ and $X_{+}\in\mathbb{R}^{n\times m_{+}}$ the matrices of the negative and positive samples, respectively, and $\mathcal{X}_{-}$ and $\mathcal{X}_{+}$ the corresponding indices sets. For notational convenience, let $e_{-}:=e_{m_{-}}$ and $e_{+}:=e_{m_{+}}$.

The binary TPMSVM approach for linear classification relies on two nonparallel hyperplanes $H_{+}$ and $H_{-}$, defined by the following equations:
\begin{linenomath}
\begin{equation*}\label{hyperplanes_linear_TPMSVM}
H_{+}: w_{+}^\top x +\theta_{+}=0 \qquad H_{-}: w_{-}^\top x +\theta_{-}=0.
\end{equation*}
\end{linenomath}
The normal vectors $w_{+},w_{-}\in\mathbb{R}^n$ and the intercepts $\theta_{+},\theta_{-}\in\mathbb{R}$ of $H_{+}$ and $H_{-}$ are solutions of the following pair of \emph{Quadratic Programming Problems} (QPPs):
\begin{linenomath}
\begin{equation} \label{positiveclass_linear_TPMSVM}
\begin{aligned}
\min_{w_+,\theta_+,\xi_+}  \quad & \frac{1}{2}\norm{w_+}_2^2+\frac{\nu_+}{m_{-}}e_{-}^{\top}\big(X_{-}^{\top}w_{+}+e_{-}\theta_{+}\big)+\frac{\alpha_+}{m_+}e_+^\top \xi_+ \\
\text{s.t.} \quad & X_+^\top w_{+}+e_{+} \theta_+ \geq - \xi_+ \\
 & \xi_+\geq 0,
\end{aligned}
\end{equation}
\end{linenomath}
and
\begin{linenomath}
\begin{equation} \label{negativeclass_linear_TPMSVM}
\begin{aligned}
\min_{w_-,\theta_-,\xi_-}  \quad & \frac{1}{2}\norm{w_-}_2^2-\frac{\nu_-}{m_{+}}e_{+}^{\top}\big(X_{+}^{\top}w_{-}+e_{+}\theta_{-}\big)+\frac{\alpha_-}{m_-}e_-^\top \xi_- \\
\text{s.t.} \quad & X_-^\top w_{-}+e_{-} \theta_- \leq \xi_- \\
 & \xi_-\geq 0,
\end{aligned}
\end{equation}
\end{linenomath}
where $\nu_{+},\nu_{-}> 0$, $\alpha_{+},\alpha_{-}>0$ are regularization parameters, balancing the terms in the objective functions. The vectors $\xi_{+}\in\mathbb{R}^{m_{+}}$, $\xi_{-}\in\mathbb{R}^{m_{-}}$ are slack vectors, associated with misclassified samples in each class (see \cite{CorVap1995}). As in the $\nu$-SVC and par-$\nu$-SVM (see \cite{SchSmoWilBar2000,Hao2010}), the ratios $\nu_{+}/\alpha_{+}$ and $\nu_{-}/\alpha_{-}$ control the fractions of supporting vectors and margin errors of each class.

The objective function of model \eqref{positiveclass_linear_TPMSVM} consists of three parts. The first term is related to the margin of the positive class, measured with respect to the $\ell_2$-norm. The second term considers the projections of negative training points on $H_{+}$, ensuring that these observations are as far as possible from $H_{+}$. Finally, the third term is the empirical risk, accounting for the total number of misclassified positive samples. Similar considerations hold for the objective function of model \eqref{negativeclass_linear_TPMSVM}.

The dual problems of models \eqref{positiveclass_linear_TPMSVM} and \eqref{negativeclass_linear_TPMSVM} are the following QPPs, respectively:
\begin{linenomath}
\begin{equation} \label{dual_positiveclass_linear_TPMSVM}
\begin{aligned}
\max_{\lambda_{+}}  \quad & -\frac{1}{2}\lambda_{+}^\top X_{+}^\top X_{+}\lambda_{+}+\frac{\nu_{+}}{m_{-}}e_{-}^\top X_{-}^\top X_{+}\lambda_{+} \\
\text{s.t.} \quad & e_{+}^\top \lambda_{+}=\nu_{+} \\
 & 0\leq \lambda_{+}\leq \frac{\alpha_{+}}{m_{+}},
\end{aligned}
\end{equation}
\end{linenomath}
and
\begin{linenomath}
\begin{equation} \label{dual_negativeeclass_linear_TPMSVM}
\begin{aligned}
\max_{\lambda_{-}}  \quad & -\frac{1}{2}\lambda_{-}^\top X_{-}^\top X_{-}\lambda_{-}+\frac{\nu_{-}}{m_{+}}e_{+}^\top X_{+}^\top X_{-}\lambda_{-} \\
\text{s.t.} \quad & e_{-}^\top \lambda_{-}=\nu_{-} \\
 & 0\leq \lambda_{-}\leq \frac{\alpha_{-}}{m_{-}},
\end{aligned}
\end{equation}
\end{linenomath}
where $\lambda_{+}\in\mathbb{R}^{m_+}$ and $\lambda_{-}\in\mathbb{R}^{m_-}$ are the Lagrangian multiplier vectors for each class. Once \eqref{dual_positiveclass_linear_TPMSVM} and \eqref{dual_negativeeclass_linear_TPMSVM} are solved, the optimal parameters $(w_{+},\theta_{+})$ and $(w_{-},\theta_{-})$ are computed through the \emph{Karush-Kuhn-Tucker} (KKT) conditions as follows:
\begin{linenomath}
\begin{equation} \label{optimal_linear_w+_b+}
w_{+} = X_{+}\lambda_{+}-\frac{\nu_{+}}{m_{-}}X_{-}e_{-} \qquad \theta_{+}=-\frac{1}{\lvert{\mathcal{N}_{+}}\rvert}\sum_{i\in\mathcal{N}_{+}}{x^i}^{\top} w_+,
\end{equation}
\end{linenomath}
and
\begin{linenomath}
\begin{equation} \label{optimal_linear_w-_b-}
w_{-} = \frac{\nu_{-}}{m_{+}}X_{+}e_{+} -X_{-}\lambda_{-} \qquad \theta_{-}=-\frac{1}{\lvert{\mathcal{N}_{-}}\rvert}\sum_{i\in\mathcal{N}_{-}}{x^i}{^\top} w_-,
\end{equation}
\end{linenomath}
with $\mathcal{N}_{+}$ the index set of positive training observations $x^i$, whose corresponding Lagrangian multiplier $\lambda_{+,i}$ satisfies $0<\lambda_{+,i}<\alpha_+/m_+$. Similarly for $\mathcal{N}_{-}$.

After the computation of $w_{+},w_{-},\theta_{+},\theta_{-}$, it is possible to classify each new observation $x\in\mathbb{R}^n$ as negative or positive according to the following decision function:
\begin{linenomath}
\begin{equation*}\label{linear_binary_classifier}
f_{\text{lin}}(x):=\sign\bigg(\frac{w_{+}^\top x + \theta_{+}}{\norm{w_{+}}_2}+\frac{w_{-}^\top x + \theta_{-}}{\norm{w_{-}}_2}\bigg).
\end{equation*}
\end{linenomath}
By way of illustration, in Figure \ref{fig_linear_binary_TPMSVM} we depict the hyperplanes $H_-$ and $H_+$, together with the final classifier $f_{\text{lin}}=0$, obtained by applying the linear TPMSVM to a bidimensional toy example. Misclassified points are represented as stars.

\begin{figure}[h!]
     \centering
     \begin{subfigure}[b]{0.496\textwidth}
         \centering
         \includegraphics[width=\textwidth]{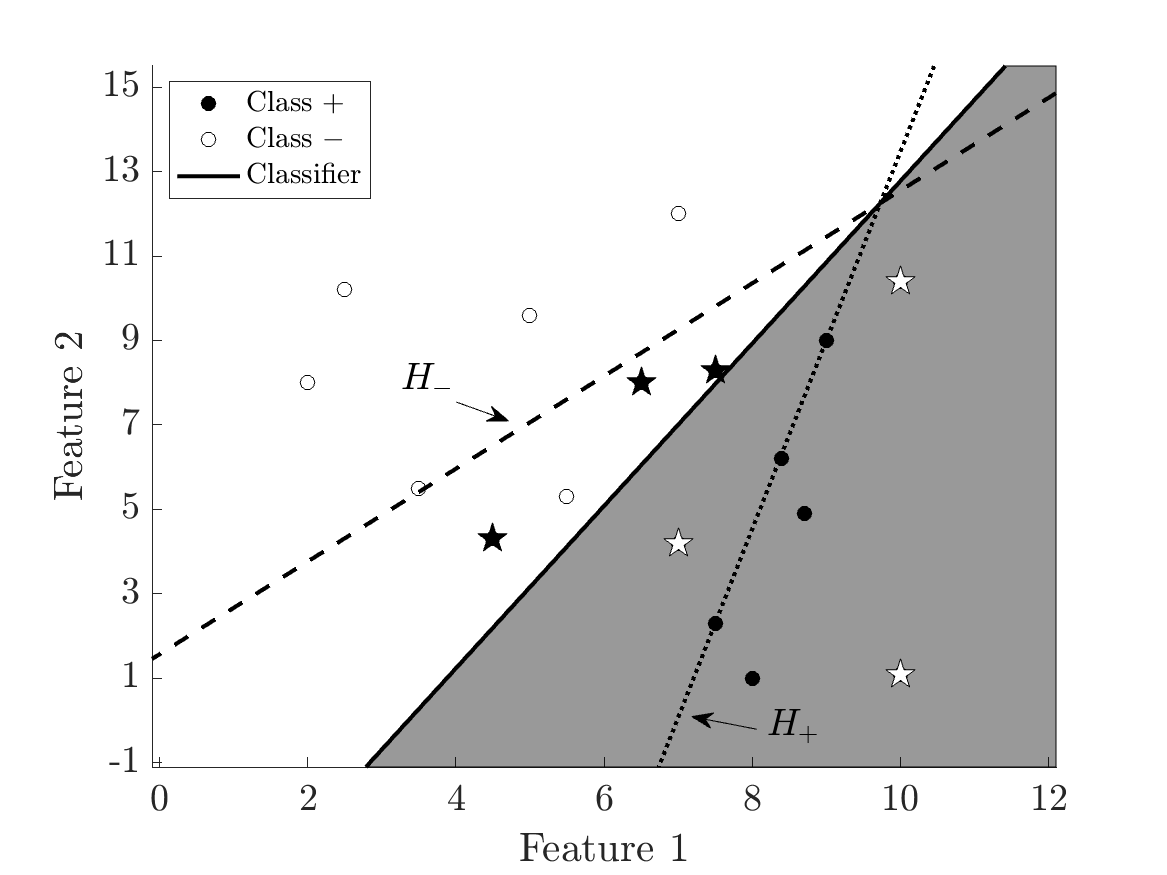}
         \caption{Linear binary TPMSVM}
         \label{fig_linear_binary_TPMSVM}
     \end{subfigure}
     \begin{subfigure}[b]{0.496\textwidth}
         \centering
         \includegraphics[width=\textwidth]{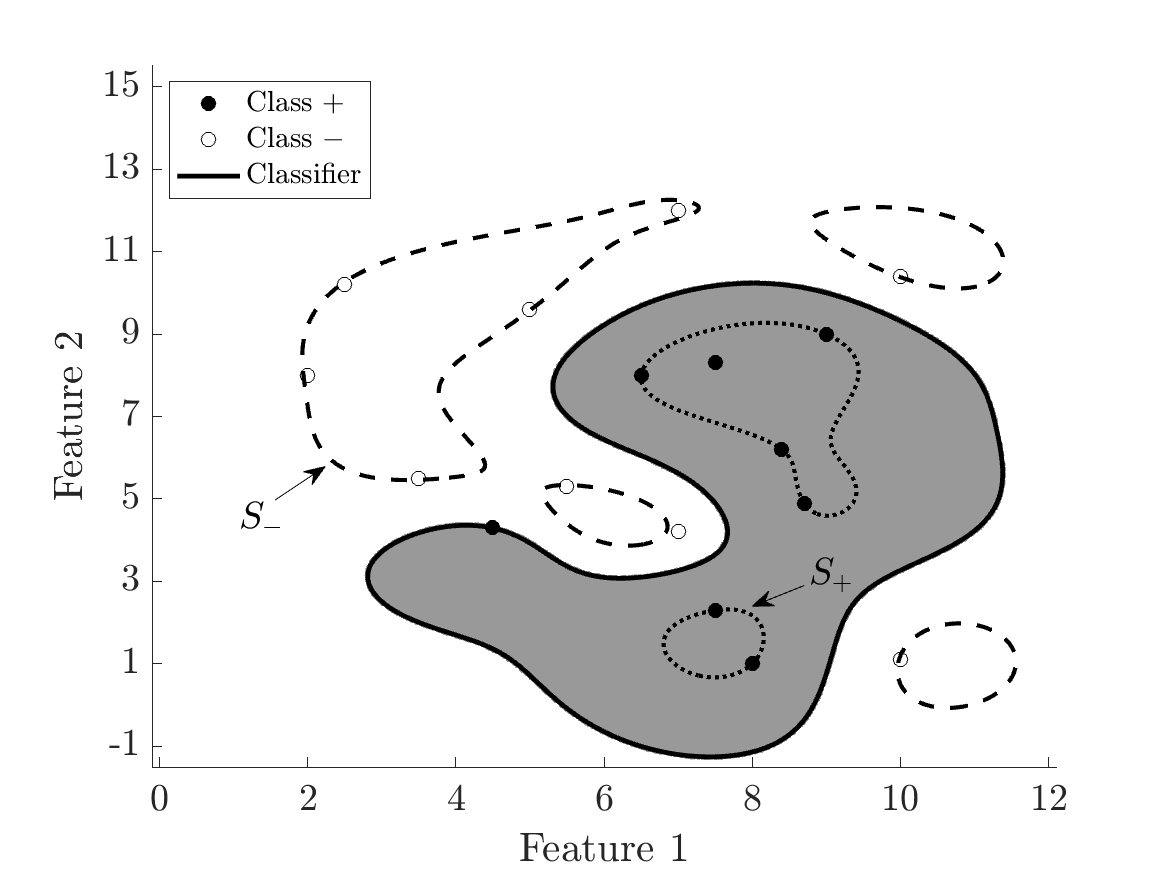}
         \caption{Nonlinear binary TPMSVM}
         \label{fig_nonlinear_binary_TPMSVM}
     \end{subfigure}
          \caption{Linear and nonlinear classifiers for the case of binary TPMSVM. The hyperparameters are $\nu_{+}=\nu_{-}=0.5$ and $\alpha_{+}=\alpha_{-}=1$. In the nonlinear case (panel on the right), the Gaussian kernel with $\sigma=1.5$ is considered. Misclassified points for each class are represented as stars.}
        \label{fig_binary_TPMSVM}
\end{figure}

\subsection{The binary TPMSVM for nonlinear classification} \label{sec_nonlinearTPMSVM}
To increase the predictive power of the model, in \cite{Peng2011} the nonlinear kernel-induced version of the TPMSVM approach is provided.

According to the classical procedure introduced in \cite{CorVap1995}, the training input data are first mapped into an inner product space ($\mathcal{H}$, $\langle \cdot,\cdot \rangle_{\mathcal{H}}$) via a map $\varphi:\mathbb{R}^n\to \mathcal{H}$. The separating hyperplanes are then defined in the feature space $\mathcal{H}$ according to the following expressions:
\begin{linenomath}
\begin{equation*}\label{hyperplanes_nonlinear_TPMSVM}
\widetilde{H}_{+}: \langle \widetilde{w}_{+}, \varphi(x) \rangle_{\mathcal{H}} +{\widetilde{\theta}_{+}}=0 \qquad \widetilde{H}_{-}: \langle \widetilde{w}_{-}, \varphi(x) \rangle_{\mathcal{H}} +{\widetilde{\theta}_{-}}=0,
\end{equation*}
\end{linenomath}
with $\widetilde{w}_{+},\widetilde{w}_{-}\in\mathcal{H}$ and ${\widetilde{\theta}_{+}},{\widetilde{\theta}_{-}}\in\mathbb{R}$.

Models \eqref{positiveclass_linear_TPMSVM} and \eqref{negativeclass_linear_TPMSVM} are modified accordingly as:
\begin{linenomath}
\begin{equation} \label{positiveclass_nonlinear_TPMSVM}
\begin{aligned}
\min_{\widetilde{w}_+,{\widetilde{\theta}_{+}},\xi_+}  \quad & \frac{1}{2}\norm{\widetilde{w}_{+}}^2_{\mathcal{H}} + \frac{\nu_+}{m_{-}}\sum_{i\in\mathcal{X}_{-}}\big(\langle \widetilde{w}_+,\varphi(x^i)\rangle_{\mathcal{H}}+{\widetilde{\theta}_{+}}\big)+\frac{\alpha_+}{m_+}e_+^\top \xi_+ \\
\text{s.t.} \quad & \langle \widetilde{w}_+,\varphi(x^i)\rangle_{\mathcal{H}} +{\widetilde{\theta}_{+}} \geq - \xi_{+,i} \qquad i \in\mathcal{X_+} \\
 & \xi_+\geq 0,
\end{aligned}
\end{equation}
\end{linenomath}
and
\begin{linenomath}
\begin{equation} \label{negativeclass_nonlinear_TPMSVM}
\begin{aligned}
\min_{\widetilde{w}_-,{\widetilde{\theta}_{-}},\xi_-}  \quad & \frac{1}{2}\norm{\widetilde{w}_{-}}^2_{\mathcal{H}} - \frac{\nu_-}{m_{+}}\sum_{i\in\mathcal{X}_{+}}\big(\langle \widetilde{w}_-,\varphi(x^i)\rangle_{\mathcal{H}}+{\widetilde{\theta}_{-}}\big)+\frac{\alpha_-}{m_-}e_-^\top \xi_- \\
\text{s.t.} \quad & \langle \widetilde{w}_-,\varphi(x^i)\rangle_{\mathcal{H}} +{\widetilde{\theta}_{-}} \leq  \xi_{-,i} \qquad i \in\mathcal{X_-} \\
 & \xi_-\geq 0,
\end{aligned}
\end{equation}
\end{linenomath}
where $\norm{\cdot}_\mathcal{H}$ is the norm induced by the dot product $\langle \cdot,\cdot \rangle_{\mathcal{H}}$, i.e. $\norm{z}_\mathcal{H}:=\sqrt{\langle z, z \rangle_{\mathcal{H}}}$, with $z\in\mathcal{H}$.

Unfortunately, since a closed-form expression of the feature map $\varphi(\cdot)$ is rarely available and the feature space $\mathcal{H}$ is potentially an infinite-dimensional space (see \cite{PicSci2018,SchSmo2001}), models \eqref{positiveclass_nonlinear_TPMSVM}-\eqref{negativeclass_nonlinear_TPMSVM} cannot be solved in practice (see \cite{Jim-CorMorPin2021}). To overcome this limitation, it is possible to reformulate their dual problems and efficiently solve them through the so-called \textit{kernel trick} (see \cite{CorVap1995}). Specifically, a symmetric and positive semidefinite kernel function $k:\mathbb{R}^n\times \mathbb{R}^n \to \mathbb{R}$ is introduced such that $k(x,x'):=\langle\varphi(x),\varphi(x')\rangle_{\mathcal{H}}$, for all $x,x'\in\mathbb{R}^n$. Examples of kernel functions typically used in the ML literature are reported in Table \ref{tab_kernels}. The reader is referred to \cite{SchSmo2001} for a comprehensive overview on kernel functions.

\begin{table}[h!]
\centering
\resizebox{0.8\textwidth}{!}{
\begin{tabular}{lll}\toprule
Kernel function & $k(x,x')$ & Parameter \\ \hline
Homogeneous polynomial & $k(x,x')= (x^{\top}x')^d$ & $d\in \mathbb{N}$\\
Inhomogeneous polynomial & $k(x,x')=(\gamma+x^{\top}x')^d$ & $\gamma \geq 0$, $d\in \mathbb{N}$\\
Gaussian & $k(x,x')= \exp\bigg(\displaystyle -\frac{\norm{x-x'}_2^2}{2 \sigma^2}\bigg)$ & $\sigma >0$\\
\bottomrule
\end{tabular}}
\caption{Examples of kernel functions. The first column reports the name of the functions. The second column provides their mathematical expressions. Finally, the third column contains their relevant parameters.} \label{tab_kernels}
\end{table}

Thus, the dual problems of models \eqref{positiveclass_nonlinear_TPMSVM} and \eqref{negativeclass_nonlinear_TPMSVM} can be reformulated in terms of the kernel function as follows:
\begin{linenomath}
\begin{equation} \label{dual_positiveclass_nonlinear_TPMSVM}
\begin{aligned}
\max_{\lambda_{+}}  \quad & -\frac{1}{2}\lambda_{+}^\top K(X_+,X_+)\lambda_{+}+\frac{\nu_{+}}{m_{-}}e_{-}^\top K(X_-,X_+)\lambda_{+} \\
\text{s.t.} \quad & e_{+}^\top \lambda_{+}=\nu_{+} \\
 & 0\leq \lambda_{+}\leq \frac{\alpha_{+}}{m_{+}},
\end{aligned}
\end{equation}
\end{linenomath}
and
\begin{linenomath}
\begin{equation} \label{dual_negativeeclass_nonlinear_TPMSVM}
\begin{aligned}
\max_{\lambda_{-}}  \quad & -\frac{1}{2}\lambda_{-}^\top K(X_-,X_-)\lambda_{-}+\frac{\nu_{-}}{m_{+}}e_{+}^\top K(X_+,X_-)\lambda_{-} \\
\text{s.t.} \quad & e_{-}^\top \lambda_{-}=\nu_{-} \\
 & 0\leq \lambda_{-}\leq \frac{\alpha_{-}}{m_{-}},
\end{aligned}
\end{equation}
\end{linenomath}
where $K(X_+,X_+)$ is the Gram matrix of the dot products $k(x^i,x^j)$ for $i,j\in\mathcal{X_{+}}$. Similarly with $K(X_+,X_-)$, $K(X_-,X_+)$ and $K(X_-,X_-)$.

As in the linear case, once problems \eqref{dual_positiveclass_nonlinear_TPMSVM} and \eqref{dual_negativeeclass_nonlinear_TPMSVM} are solved, the KKT conditions provide $(\widetilde{w}_+,{\widetilde{\theta}_{+}})$ and $(\widetilde{w}_-,{\widetilde{\theta}_{-}})$, defining the hyperplanes $\widetilde{H}_{+}$ and $\widetilde{H}_{-}$ in the feature space $\mathcal{H}$. These hyperplanes, in turn, induce nonlinear decision boundaries $S_{+}$ and $S_{-}$ in the input space $\mathbb{R}^n$.

Finally, the decision function in the case of binary TPMSVM with nonlinear classifiers is given as follows:
\begin{linenomath}
\begin{equation*}\label{nonlinear_binary_classifier}
f_{\text{nonlin}}(x):=\sign\bigg(\frac{\langle \widetilde{w}_+,\varphi(x) \rangle_{\mathcal{H}} + {\widetilde{\theta}_{+}}}{\norm{\widetilde{w}_+}_{\mathcal{H}}}+\frac{\langle \widetilde{w}_-,\varphi(x) \rangle_{\mathcal{H}} + {\widetilde{\theta}_{-}}}{\norm{\widetilde{w}_+}_{\mathcal{H}}}\bigg).
\end{equation*}
\end{linenomath}

Figure \ref{fig_nonlinear_binary_TPMSVM} shows the separating hypersurfaces $S_+$ and $S_-$ induced by a Gaussian kernel when classifying the same bidimensional toy example of Figure \ref{fig_linear_binary_TPMSVM}.

\section{A novel multiclass TPMSVM-type model} \label{sec_multiclass_deterministic}
In this section, we extend the binary TPMSVM approach recalled so far to address multiclass classification problems, considering both linear (Section \ref{sec_linear_multiclass}) and nonlinear (Section \ref{sec_nonlinear_multiclass}) decision boundaries. Among all the possible multiclass strategies, we opt for the one-versus-all formulation due to its lower computational complexity and strong performance in terms of accuracy (see \cite{DingZhaZhaZhaXue2019}).

In the following, we assume that $C$ represents the total number of classes and, for each class $c=1,\ldots,C$, the subscript $\cdot_{-c}$ denotes the set of points not belonging to class $c$.

\subsection{The multiclass TPMSVM for linear classification} \label{sec_linear_multiclass}
Let $\{x^i,y_i\}_{i=1}^m$ be the set of training samples, with $y_i\in\{1,\ldots,C\}$. For each class $c$, with $c=1,\ldots,C$, we denote by $m_c$ the number of observations belonging to class $c$ and $m_{-c}:=m-{m_c}$. Matrix $X_c\in\mathbb{R}^{n\times m_c}$ comprises all training data points of class $c$, and correspondingly, matrix $X_{-c}\in\mathbb{R}^{n\times m_{-c}}$ includes the data for all the other classes. The associated indices sets are $\mathcal{X}_c$ and $\mathcal{X}_{-c}$, respectively, and $\mathcal{X}:=\mathcal{X}_{c}\cup \mathcal{X}_{-c}$. Let $e_c:=e_{m_{c}}$ and $e_{-c}:=e_{m_{-c}}$.

According to the one-versus-all strategy, for each class $c=1,\ldots,C$ we aim to find the best separating hyperplane $H_c$ defined by equation $w_c^\top x + \theta_c=0$, where $w_c \in\mathbb{R}^n$ and $\theta_c \in \mathbb{R}$ are solutions of the following QPP:
\begin{linenomath}
\begin{equation} \label{multiclass_linear_TPMSVM}
\begin{aligned}
\min_{w_c,\theta_c,\xi_c}  \quad & \frac{1}{2}\norm{w_c}_2^2+\frac{\nu_c}{m_{-c}}e_{-c}^{\top}\big(X_{-c
}^{\top}w_c+e_{-c}\theta_c\big)+\frac{\alpha_c}{m_c}e_c^\top \xi_c \\
\text{s.t.} \quad & X_c^\top w_c+e_c \theta_c \geq - \xi_c \\
 & \xi_c\geq 0.
\end{aligned}
\end{equation}
\end{linenomath}
Parameters $\nu_c>0$, $\alpha_c>0$ and slack vector $\xi_c\in\mathbb{R}^{m_c}$ have an equivalent interpretation of the corresponding terms in model \eqref{positiveclass_linear_TPMSVM}.

By introducing the Lagrangian function of problem \eqref{multiclass_linear_TPMSVM}, the KKT conditions lead to closed-form expressions for the normal vector $w_c$ and the intercept $\theta_c$ of the hyperplane $H_c$, for all $c=1,\ldots,C$, following a similar approach to the one used to derive \eqref{optimal_linear_w+_b+}.

Once all the $C$ hyperplanes have been determined, we propose two possible alternatives for the decision function:
\begin{linenomath} 
\begin{equation} \label{dec_func_multiclass_linear_argmin}
f_{\text{lin},\min}(x):=\argmin_{c=1,\ldots,C} \frac{\abs{w_c^\top x +\theta_c}}{\norm{w_c}_2},
\end{equation}
\end{linenomath}
and
\begin{linenomath} 
\begin{equation} \label{dec_func_multiclass_linear_argmax}
f_{\text{lin},\max}(x):=\argmax_{c=1,\ldots,C} \frac{w_c^\top x +\theta_c}{\norm{w_c}_2}.
\end{equation}
\end{linenomath}

In Figures \ref{fig_linear_multiclass_TPMSVM_argmin}-\ref{fig_linear_multiclass_TPMSVM_argmax} we represent the results of the proposed linear methodology in the case of a multiclass classification task. We consider a bidimensional toy example with three classes (black, grey, and white points, respectively). The hyperplanes $H_1$, $H_2$, $H_3$ are depicted, along with the decision functions computed according to formulas \eqref{dec_func_multiclass_linear_argmin} and \eqref{dec_func_multiclass_linear_argmax}.

\begin{figure}[h!]
     \centering
     \begin{subfigure}[b]{0.496\textwidth}
         \centering
         \includegraphics[width=\textwidth]{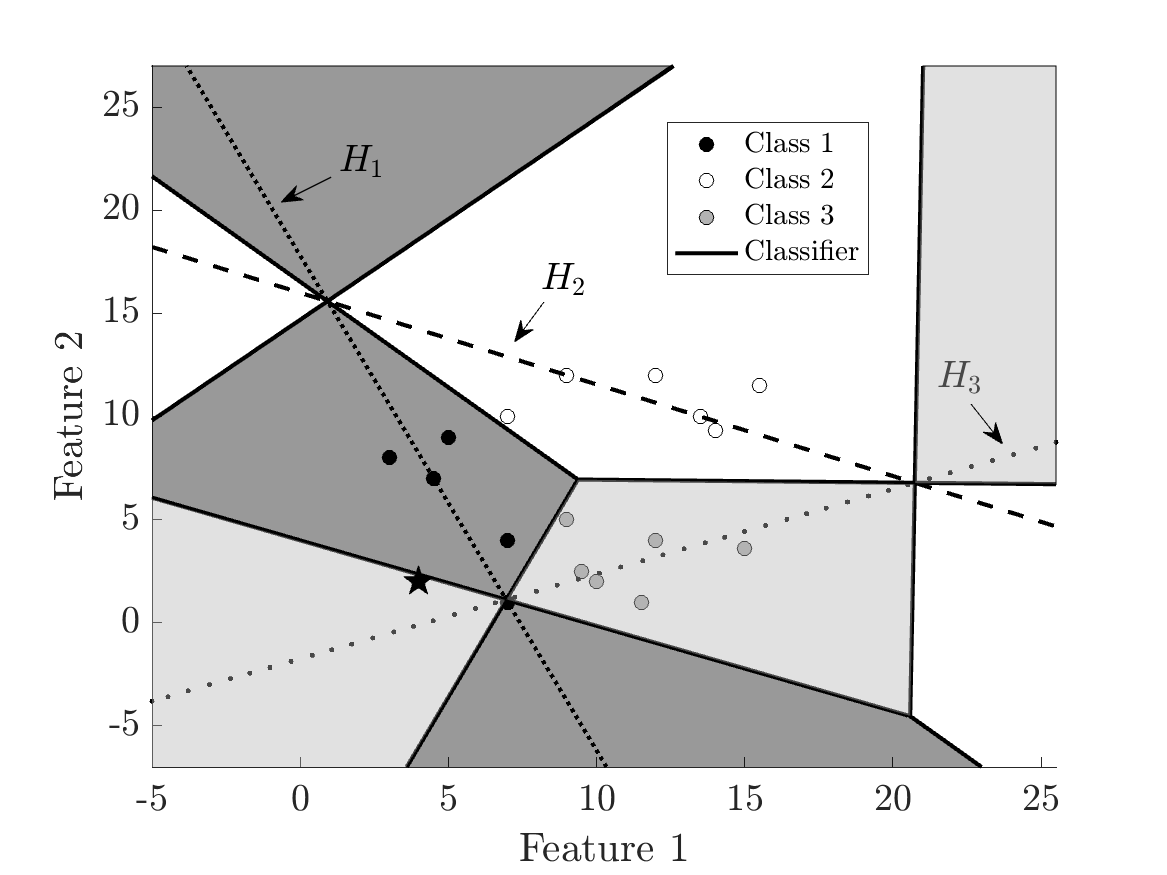}
         \caption{Linear multiclass TPMSVM with argmin formula \eqref{dec_func_multiclass_linear_argmin}}
         \label{fig_linear_multiclass_TPMSVM_argmin}
     \end{subfigure}
     \begin{subfigure}[b]{0.496\textwidth}
         \centering
         \includegraphics[width=\textwidth]{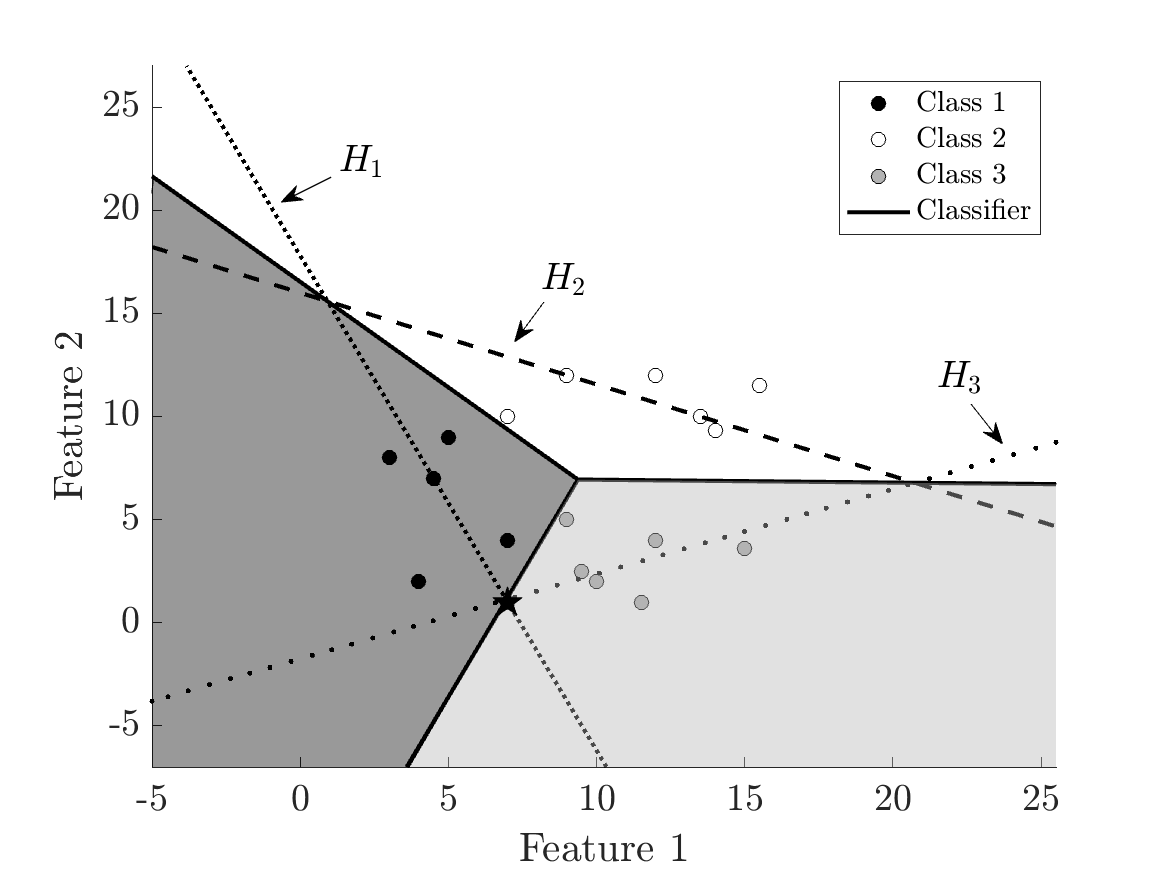}
         \caption{Linear multiclass TPMSVM with argmax formula \eqref{dec_func_multiclass_linear_argmax}}
         \label{fig_linear_multiclass_TPMSVM_argmax}
     \end{subfigure}
     \par\bigskip
          \begin{subfigure}[b]{0.496\textwidth}
         \centering
         \includegraphics[width=\textwidth]{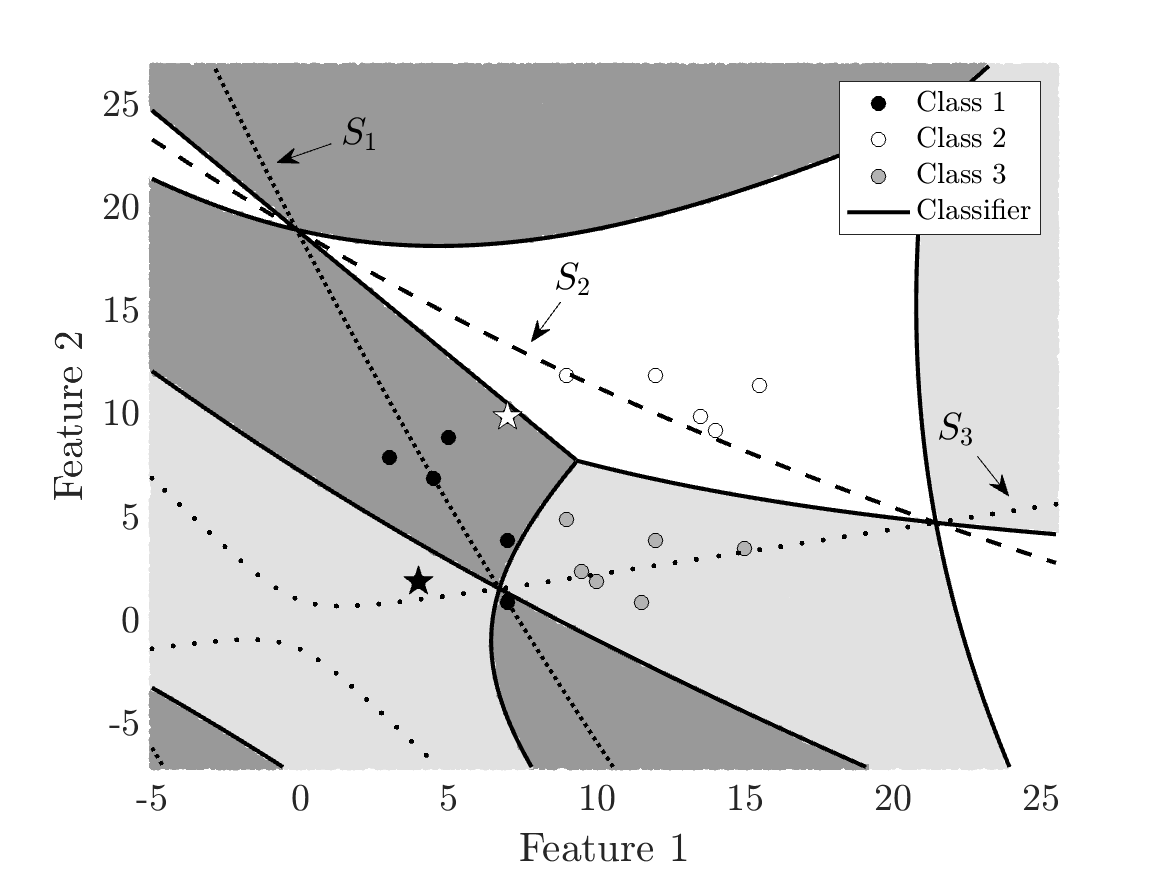}
         \caption{Nonlinear multiclass TPMSVM with argmin formula \eqref{dec_func_multiclass_nonlinear_argmin}}
         \label{fig_nonlinear_multiclass_TPMSVM_argmin}
     \end{subfigure}
     \begin{subfigure}[b]{0.496\textwidth}
         \centering
         \includegraphics[width=\textwidth]{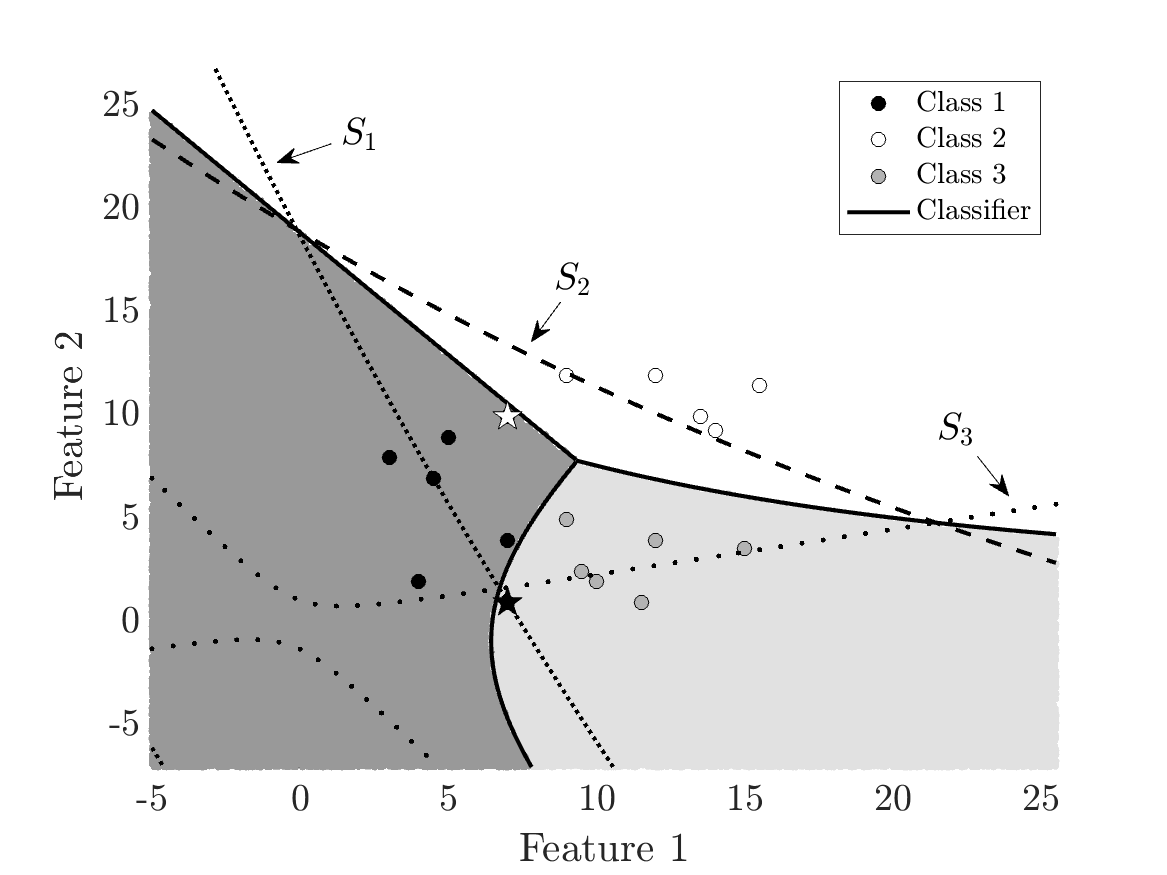}
         \caption{Nonlinear multiclass TPMSVM with argmax formula \eqref{dec_func_multiclass_nonlinear_argmax}}
         \label{fig_nonlinear_multiclass_TPMSVM_argmax}
     \end{subfigure}
          \caption{Linear and nonlinear classifiers for the case of three-classes TPMSVM. The hyperparameters are $\nu_{c}=0.5$ and $\alpha_{c}=1$ for $c=1,2,3$. In the nonlinear case (panels (c) and (d)), the inhomogeneous polynomial kernel with $d=2$ and $\gamma=1.5$ is considered. Misclassified points for each class are represented as stars.}
        \label{fig_multiclass_TPMSVM}
\end{figure}

\subsection{The multiclass TPMSVM for nonlinear classification} \label{sec_nonlinear_multiclass}
When dealing with nonlinear classifiers, model \eqref{multiclass_linear_TPMSVM} is defined in the feature space $\mathcal{H}$, leading to the following formulation:
\begin{linenomath}
\begin{equation} \label{multiclass_nonlinear_TPMSVM}
\begin{aligned}
\min_{\widetilde{w}_c,{\widetilde{\theta}_c},\xi_c}  \quad & \frac{1}{2}\norm{\widetilde{w}_c}_{\mathcal{H}}^2+\frac{\nu_c}{m_{-c}} \sum_{i\in\mathcal{X}_{-c}} (\langle \widetilde{w}_c, \varphi(x^i)\rangle_{\mathcal{H}} + {\widetilde{\theta}_c} )+\frac{\alpha_c}{m_c}\sum_{i \in \mathcal{X}_{c}} \xi_{c,i} \\
\text{s.t.} \quad & \langle \widetilde{w}_c, \varphi (x^i) \rangle_{\mathcal{H}} + {\widetilde{\theta}_c} \geq - \xi_{c,i} \qquad i \in \mathcal{X}_c \\
 & \xi_{c,i}\geq 0 \qquad i \in \mathcal{X}_c.
\end{aligned}
\end{equation}
\end{linenomath}

As discussed in Section \ref{sec_nonlinearTPMSVM}, the KKT conditions provide the optimal solutions $(\widetilde{w}_c,{\widetilde{\theta}_c})$ in terms of the Lagrangian multiplier $\lambda_c\in\mathbb{R}^{m_c}$ as follows:
\begin{linenomath}
\begin{equation} \label{optimal_wc_bc_nonlinear}
\widetilde{w}_{c} = \sum_{i\in\mathcal{X}_{c}} \lambda_{c,i} \varphi(x^i)-\frac{\nu_{c}}{m_{-c}}\sum_{i\in\mathcal{X}_{-c}} \varphi(x^i) \qquad {\widetilde{\theta}_c}=-\frac{1}{\lvert{\mathcal{N}_{c}}\rvert}\sum_{i\in\mathcal{N}_{c}} \langle \varphi(x^i),\widetilde{w}_c \rangle_{\mathcal{H}},
\end{equation}
\end{linenomath}
where $\mathcal{N}_{c}$ is the index set of observations $x^i$, with $i\in\mathcal{X}_c$, whose corresponding Lagrangian multiplier satisfies $0<\lambda_{c,i}<\alpha_c/m_c$.

Within the nonlinear context, we propose as decision functions the following:
\begin{linenomath} 
\begin{equation} \label{dec_func_multiclass_nonlinear_argmin}
f_{\text{nonlin},\min}(x):=\argmin_{c=1,\ldots,C} \frac{\abs{\langle \widetilde{w}_c, \varphi(x) \rangle_{\mathcal{H}} +{\widetilde{\theta}_c}}}{\norm{\widetilde{w}_c}_{\mathcal{H}}},
\end{equation} 
\end{linenomath}
and
\begin{linenomath} 
\begin{equation} \label{dec_func_multiclass_nonlinear_argmax}
f_{\text{nonlin},\max}(x):=\argmax_{c=1,\ldots,C} \frac{{\langle \widetilde{w}_c, \varphi(x) \rangle_{\mathcal{H}} +{\widetilde{\theta}_c}}}{\norm{\widetilde{w}_c}_{\mathcal{H}}}.
\end{equation}
\end{linenomath}
For all $c=1,\ldots,C$, each term in \eqref{dec_func_multiclass_nonlinear_argmin}-\eqref{dec_func_multiclass_nonlinear_argmax} is computed using the kernel trick, resulting in:
\begin{linenomath}
\begin{equation*}
\langle \widetilde{w}_c, \varphi(x) \rangle_{\mathcal{H}} = \lambda_c^\top K(X_c,x)-\frac{\nu_c}{m_{-c}}e_{-c}^\top K(X_{-c},x),
\end{equation*}
\end{linenomath}
\begin{linenomath}
\begin{equation*}
\widetilde{\theta}_c = -\frac{1}{\lvert{\mathcal{N}_{c}}\rvert}\sum_{i\in\mathcal{N}_{c}} \big[K(x^i,X_c)\lambda_c-\frac{\nu_c}{m_{-c}}K(x^i,X_{-c})e_{-c}\big],
\end{equation*}
\end{linenomath}
and
\begin{linenomath}
\begin{equation*}
\begin{split}
\norm{\widetilde{w}_c}^2_{\mathcal{H}} = \langle {\widetilde{w}_c}, {\widetilde{w}_c} \rangle_{\mathcal{H}} &= \lambda_c^{\top}K(X_c,X_c)\lambda_c-\frac{\nu_c}{m_{-c}}\lambda_c^\top K(X_c,X_{-c})e_{-c}+\\
								&-\frac{\nu_c}{m_{-c}} e_{-c}^\top K(X_{-c},X_{c})\lambda_c
								+\frac{\nu_c^2}{m_{-c}^2}e_{-c}^\top K(X_{-c},X_{-c})e_{-c}.
\end{split}
\end{equation*}
\end{linenomath}

The computational complexity of the binary TPMSVM is approximately $\mathcal{O}(m^3/4)$, as it involves solving two QPPs, each of size $m/2$ (see \cite{Peng2011}). In the proposed multiclass extension, one QPP is solved for each of the $C$ classes, both in the linear and kernelized settings. As a result, the overall computational complexity of our multiclass TPMSVM approach is $C \cdot \mathcal{O}(m^3/8)$.

In Figures \ref{fig_nonlinear_multiclass_TPMSVM_argmin}-\ref{fig_nonlinear_multiclass_TPMSVM_argmax} we illustrate the results of model \eqref{multiclass_nonlinear_TPMSVM}, considering the same bidimensional toy dataset of Figures \ref{fig_linear_multiclass_TPMSVM_argmin}-\ref{fig_linear_multiclass_TPMSVM_argmax}. The hypersurfaces $S_1,S_2,S_3$ are induced by an inhomogeneous quadratic kernel, and the decision functions are computed according to formulas \eqref{dec_func_multiclass_nonlinear_argmin} and \eqref{dec_func_multiclass_nonlinear_argmax}.

\section{The robust multiclass TPMSVM-type model} \label{sec_multiclass_robust}
In this section, we derive the robust counterpart formulations of models \eqref{multiclass_linear_TPMSVM} and \eqref{multiclass_nonlinear_TPMSVM}. Specifically, we start by constructing bounded-by-$\ell_p$-norm uncertainty sets around training data points. The deterministic approaches presented in Section \ref{sec_multiclass_deterministic} are then robustified by optimizing over the worst-case realization of the uncertain data across the entire uncertainty sets. Tractable SOCP reformulations are finally provided.

Section \ref{sec_rob_lin_multiclass} focuses on the robust extension of multiclass TPMSVM for linear classification, while Section \ref{sec_rob_nonlin_multiclass} explores cases with kernel-induced decision boundaries.

\subsection{The robust TPMSVM for linear multiclass classification} \label{sec_rob_lin_multiclass}
As in \cite{MagSpi2023}, in  the construction of the uncertainty set we assume that each training observation $x^i\in\mathbb{R}^n$ is subject to an unknown bounded-by-$\ell_p$-norm perturbation $\delta^i\in\mathbb{R}^n$, with $p\in[1,\infty]$. Therefore, the uncertainty set around $x^i$ can be written as:
\begin{linenomath}
\begin{equation} \label{uncertainty_set_input_space}
\mathcal{U}_p(x^i):=\big\{x\in\mathbb{R}^n \lvert x=x^i+\delta^i, \norm{\delta^i}_p\leq \varepsilon_i\big\}.
\end{equation}
\end{linenomath}
The radius $\varepsilon_i\geq 0$ controls the degree of conservatism: when $\varepsilon_i=0$, the uncertainty set $\mathcal{U}_p(x^i)$ reduces to its center $x^i$.

Robustifying model \eqref{multiclass_linear_TPMSVM} against the uncertainty set $\mathcal{U}_p(x^i)$ yields the following optimization model:
\begin{linenomath}
\begin{equation} \label{robust_linear_intractable}
\begin{aligned}
\min_{w_c,\theta_c,\xi_c}  \quad & \frac{1}{2}\norm{w_c}_2^2+\frac{\nu_c}{m_{-c}} \sum_{i\in\mathcal{X}_{-c}} \max_{\norm{\delta^i}_p\leq \varepsilon_i} \big[\big(x^{i^\top}+\delta^{i^\top}\big) w_c+\theta_c\big]+\frac{\alpha_c}{m_c}\sum_{i\in\mathcal{X}_{c}} \xi_{c,i} \\
\text{s.t.} \quad & x^\top w_c+\theta_c \geq -\xi_{c,i} \qquad i\in\mathcal{X}_{c}, \forall x\in \mathcal{U}_{p}(x^i) \\
 & \xi_{c,i} \geq 0 \qquad i \in \mathcal{X}_{c}.
\end{aligned}
\end{equation}
\end{linenomath}

Since there exist infinite possibilities for choosing $x\in\mathcal{U}_{p}(x^i)$ in the first set of constraints, model \eqref{robust_linear_intractable} is not solvable in practice. However, a tractable closed-form reformulation is provided in the following theorem.

\begin{theorem}
Let $\mathcal{U}_{p}(x^i)$ be the uncertainty set as in \eqref{uncertainty_set_input_space}, with $p\in[1,\infty]$. Let $p'$ be the H\"older conjugate of $p$, namely $1/p+1/p'=1$. Model \eqref{robust_linear_intractable} is equivalent to:
\begin{linenomath}
\begin{equation} \label{robust_linear_tractable}
\begin{aligned}
\min_{w_c,\theta_c,\xi_c}  \quad & \frac{1}{2}\norm{w_c}_2^2+\frac{\nu_c}{m_{-c}} \sum_{i\in\mathcal{X}_{-c}} \big(x^{i^\top} w_c+\varepsilon_i\norm{w_c}_{p'}\big)+\nu_c \theta_c +\frac{\alpha_c}{m_c}\sum_{i\in\mathcal{X}_{c}} \xi_{c,i} \\
\text{s.t.} \quad & x^{i^\top} w_c+\theta_c-\varepsilon_i \norm{w_c}_{p'} \geq -\xi_{c,i} \qquad i\in\mathcal{X}_{c} \\
 & \xi_{c,i} \geq 0 \qquad i \in \mathcal{X}_{c}.
\end{aligned}
\end{equation}
\end{linenomath}
\end{theorem}

\begin{proof}
Following the RO approach of \cite{BerDunPawZhu2019}, model \eqref{robust_linear_intractable} can be expressed as follows:
\begin{linenomath}
\begin{equation} \label{robust_linear_minmax}
\begin{aligned}
\min_{w_c,\theta_c,\xi_c}  \quad & \frac{1}{2}\norm{w_c}_2^2+\frac{\nu_c}{m_{-c}} \sum_{i\in\mathcal{X}_{-c}} \max_{\norm{\delta^i}_p\leq \varepsilon_i} \big[\big(x^{i^\top}+\delta^{i^\top}\big) w_c+\theta_c\big]+\frac{\alpha_c}{m_c}\sum_{i\in\mathcal{X}_{c}} \xi_{c,i} \\
\text{s.t.} \quad & \min_{\norm{\delta^i}_p\leq \varepsilon_i} \big[\big(x^{i^\top}+\delta^{i^\top}\big) w_c\big]+\theta_c \geq -\xi_{c,i} \qquad i\in\mathcal{X}_{c} \\
 & \xi_{c,i} \geq 0 \qquad i \in \mathcal{X}_{c}.
\end{aligned}
\end{equation}
\end{linenomath}
The maximization term in the objective function corresponds to:
\begin{linenomath}
\begin{equation} \label{mimax_objfunc_linear}
\frac{\nu_c}{m_{-c}} \sum_{i\in\mathcal{X}_{-c}} \max_{\norm{\delta^i}_p\leq \varepsilon_i} \big[x^{i^\top}w_c+\delta^{i^\top} w_c+\theta_c\big] = \nu_c\theta_c+ \frac{\nu_c}{m_{-c}} \sum_{i\in\mathcal{X}_{-c}} \bigg( x^{i^\top} w_c+ \max_{\norm{\delta^i}_p\leq \varepsilon_i} \big[\delta^{i^\top}w_c\big]\bigg).
\end{equation}
\end{linenomath}
Similarly, the minimization term in the first set of constraints is equivalent to:
\begin{equation} \label{minmax_constr_linear}
\min_{\norm{\delta^i}_p\leq \varepsilon_i} \big[x^{i^\top}w_c+\delta^{i^\top}w_c\big] = x^{i^\top}w_c+\min_{\norm{\delta^i}_p\leq \varepsilon_i} \big[\delta^{i^\top}w_c\big].
\end{equation}
We observe that the optimization problems in \eqref{mimax_objfunc_linear} and in \eqref{minmax_constr_linear} have the same feasible set, $\{\delta^i:\norm{\delta^i}_p\leq \varepsilon_i\}$, and the same objective function, $\delta^{i^\top}w_c$. By applying the H\"older inequality (see \cite{Rud1987}), we get:
\begin{linenomath}
\begin{equation} \label{deltawc_proof}
-\varepsilon_i \norm{w_c}_{p'} \leq \delta^{i^\top}w_c \leq \varepsilon_i \norm{w_c}_{p'} \qquad \forall \delta^i:\norm{\delta^i}_p\leq \varepsilon_i,
\end{equation}
\end{linenomath}
where $p'$ is the H\"older conjugate of $p$. Consequently, the optimization problems in \eqref{mimax_objfunc_linear} and in \eqref{minmax_constr_linear} have the following optimal values:
\begin{equation}
\max_{\norm{\delta^i}_p\leq \varepsilon_i} \big[\delta^{i^\top}w_c\big]=\varepsilon_i \norm{w_c}_{p'} \quad \text{and} \quad \min_{\norm{\delta^i}_p\leq \varepsilon_i} \big[\delta^{i^\top}w_c\big]=-\varepsilon_i \norm{w_c}_{p'}.
\end{equation}
Therefore, the second term in the objective function of \eqref{robust_linear_minmax} corresponds to:
\begin{linenomath}
\begin{equation*}
\nu_c\theta_c+ \frac{\nu_c}{m_{-c}} \sum_{i\in\mathcal{X}_{-c}} \bigg( x^{i^\top} w_c+ \varepsilon_i\norm{w_c}_{p'}\bigg),
\end{equation*}
\end{linenomath}
while the first set of constraints can be written as follows:
\begin{linenomath}
\begin{equation*}
x^{i^\top}w_c + \theta_c -\varepsilon_i \norm{w_c}_{p'} \geq -\xi_{c,i}, \quad i \in \mathcal{X}_c.
\end{equation*}
\end{linenomath}
This concludes the proof.
\end{proof}

If no uncertainty occurs in the training samples, $\varepsilon_i=0$ for all $i\in\mathcal{X}$ and the robust model \eqref{robust_linear_tractable} reduces to the deterministic model \eqref{multiclass_linear_TPMSVM}. We notice that model \eqref{robust_linear_tractable} is a convex nonlinear optimization model due to the presence of the $\ell_{2}$- and $\ell_{p'}$-norm of $w_c$. However, the quadratic term $\norm{w_c}_2^2$ can be easily transformed from the objective function to the constraints by introducing auxiliary variables $t_c$, $u_c$, $v_c\in\mathbb{R}$ (see \cite{QiTianShi2013}), leading to:
\begin{linenomath}
\begin{equation} \label{robust_linear_SOCP_p'}
\begin{aligned}
\min_{w_c,\theta_c,\xi_c,t_c,u_c,v_c}  \quad & \frac{1}{2}(u_c-v_c)+\frac{\nu_c}{m_{-c}} \sum_{i\in\mathcal{X}_{-c}} \big(x^{i^\top} w_c+\varepsilon_i\norm{w_c}_{p'}\big)+\nu_c \theta_c +\frac{\alpha_c}{m_c}\sum_{i\in\mathcal{X}_{c}} \xi_{c,i} \\
\text{s.t.} \quad & x^{i^\top} w_c+\theta_c-\varepsilon_i \norm{w_c}_{p'} \geq -\xi_{c,i} \qquad i\in\mathcal{X}_{c} \\
& t_c \geq \norm{w_c}_2\\
& u_c+v_c = 1\\
& u_c \geq \sqrt{t_c^2+v_c^2}\\
 & \xi_{c,i} \geq 0 \qquad i \in \mathcal{X}_{c}.
\end{aligned}
\end{equation}
\end{linenomath}

Model \eqref{robust_linear_SOCP_p'} is still nonlinear, but for specific choices of the $\ell_{p}$-norm used in the definition of the uncertainty set \eqref{uncertainty_set_input_space}, it can be reformulated as a SOCP model, as stated in the following result.

\begin{corollary} \label{corollary_robust_linear}
Let $\mathcal{U}_{p}(x^i)$ be the uncertainty set as in \eqref{uncertainty_set_input_space}. Model \eqref{robust_linear_SOCP_p'} can be expressed as a SOCP model in the following cases:
\begin{itemize}
\item[a)] Case $p=1$:
\begin{linenomath}
\begin{equation} \label{robust_linear_SOCP_polyhedral}
\begin{aligned}
\min_{w_c,\theta_c,\xi_c,t_c,u_c,v_c,s_c}  \quad & \frac{1}{2}(u_c-v_c)+\frac{\nu_c}{m_{-c}} \sum_{i\in\mathcal{X}_{-c}} \big(x^{i^\top} w_c+\varepsilon_i s_c\big)+\nu_c \theta_c +\frac{\alpha_c}{m_c}\sum_{i\in\mathcal{X}_{c}} \xi_{c,i} \\
\text{s.t.} \quad & x^{i^\top} w_c+\theta_c-\varepsilon_i s_c \geq -\xi_{c,i} \qquad i\in\mathcal{X}_{c} \\
& t_c \geq \norm{w_c}_2\\
& u_c+v_c = 1\\
& u_c \geq \sqrt{t_c^2+v_c^2}\\
& s_{c} \geq -w_{c,j} \qquad j=1,\ldots,n \\
& s_{c} \geq w_{c,j} \qquad j=1,\ldots,n \\
 & s_{c} \geq 0 \\
 & \xi_{c,i} \geq 0 \qquad i \in \mathcal{X}_{c}.
\end{aligned}
\end{equation}
\end{linenomath}

\item[b)] Case $p=2$:
\begin{linenomath}
\begin{equation} \label{robust_linear_SOCP_spherical}
\begin{aligned}
\min_{w_c,\theta_c,\xi_c,t_c,u_c,v_c}  \quad & \frac{1}{2}(u_c-v_c)+\frac{\nu_c}{m_{-c}} \sum_{i\in\mathcal{X}_{-c}} \big(x^{i^\top} w_c+\varepsilon_i t_c\big)+\nu_c \theta_c +\frac{\alpha_c}{m_c}\sum_{i\in\mathcal{X}_{c}} \xi_{c,i}\\
\text{s.t.} \quad & x^{i^\top} w_c+\theta_c-\varepsilon_i t_c \geq -\xi_{c,i} \qquad i\in\mathcal{X}_{c} \\
& t_c \geq \norm{w_c}_2\\
& u_c+v_c = 1\\
& u_c \geq \sqrt{t_c^2+v_c^2}\\
 & \xi_{c,i} \geq 0 \qquad i \in \mathcal{X}_{c}.
\end{aligned}
\end{equation}
\end{linenomath}

\item[c)] Case $p=\infty$:
\begin{linenomath}
\begin{equation} \label{robust_linear_SOCP_box}
\begin{aligned}
\min_{w_c,\theta_c,\xi_c,t_c,u_c,v_c,s_c}  \quad & \frac{1}{2}(u_c-v_c)+\frac{\nu_c}{m_{-c}} \sum_{i\in\mathcal{X}_{-c}} \big(x^{i^\top} w_c+\varepsilon_i \sum_{j=1}^n s_{c,j}\big)+\nu_c \theta_c +\frac{\alpha_c}{m_c}\sum_{i\in\mathcal{X}_{c}} \xi_{c,i} \\
\text{s.t.} \quad & x^{i^\top} w_c+\theta_c-\varepsilon_i \sum_{j=1}^n s_{c,j} \geq -\xi_{c,i} \qquad i\in\mathcal{X}_{c} \\
& t_c \geq \norm{w_c}_2\\
& u_c+v_c = 1\\
& u_c \geq \sqrt{t_c^2+v_c^2}\\
& s_{c,j} \geq -w_{c,j} \qquad j=1,\ldots,n \\
& s_{c,j} \geq w_{c,j} \qquad j=1,\ldots,n \\
 & s_{c,j} \geq 0 \qquad j=1,\ldots,n \\
 & \xi_{c,i} \geq 0 \qquad i \in \mathcal{X}_{c} .
\end{aligned}
\end{equation}
\end{linenomath}

\end{itemize}
\end{corollary}

\begin{proof}

\begin{itemize}
\item[a)] If $p=1$, then $p'=\infty$. By introducing an auxiliary variable $s_c \geq 0$ equal to $\norm{w_c}_{\infty}$, and adding the constraints $s_c \geq -w_{c,j}$ and $s_c \geq w_{c,j}$ for all $j=1,\ldots,n$, model \eqref{robust_linear_SOCP_p'} is equivalent to model \eqref{robust_linear_SOCP_polyhedral}.
\item[b)] If $p=2$, then $p'=2$, and so $\norm{w_c}_{p'}=\norm{w_c}_{2}=t_c$. The equivalence between models \eqref{robust_linear_SOCP_p'} and \eqref{robust_linear_SOCP_spherical} follows straightforwardly.
\item[c)] If $p=\infty$, then $p'=1$. In this case model \eqref{robust_linear_SOCP_p'} can be rewritten as model \eqref{robust_linear_SOCP_box} by introducing an auxiliary vector $s_c \in \mathbb{R}^n$ such that each component $s_{c,j}$ is equal to $\abs{w_{c,j}}$ and adding the constraints $s_{c,j} \geq 0$, $s_{c,j} \geq - w_{c,j}$ and $s_{c,j} \geq w_{c,j}$ for all $j=1,\ldots,n$.
\end{itemize}
\end{proof}

As in the deterministic setting (see Section \ref{sec_linear_multiclass}), once the optimal solutions $(w_c,\theta_c)$ are obtained for all $c=1,\ldots,C$, the classification of each new observation is performed according to decision functions \eqref{dec_func_multiclass_linear_argmin}-\eqref{dec_func_multiclass_linear_argmax}.

SOCP models are typically solved using specialized algorithms. One of the most widely used is the \emph{primal-dual interior-point method}, whose computational complexity is bounded by $\mathcal{O}\left(N^{7/2} \log(1/\eta)\right)$, where $N$ denotes the total number of variables and $\eta$ is the duality gap threshold that determines the solution accuracy (see \cite{wright1997primal}). For the robust models in problems \eqref{robust_linear_SOCP_polyhedral}, \eqref{robust_linear_SOCP_spherical}, and \eqref{robust_linear_SOCP_box}, the corresponding problem sizes are  $N = n + m_c + 5 $, $ N = n + m_c + 4 $, and $N = 2n + m_c + 4 $, respectively. Since a SOCP model is solved independently for each class, the overall time complexity scales with the number of classes $C$, yielding a bound on the total computational complexity of $C \cdot \mathcal{O}\left(N^{7/2} \log(1/\eta)\right)$ for each formulation. Note that, since $m_c \leq m$, we can substitute $m_c$ with $m$ in each expression of $N$ to obtain a valid upper bound on the total complexity.

\subsection{The robust TPMSVM for nonlinear multiclass classification}  \label{sec_rob_nonlin_multiclass}
Let $k:\mathbb{R}^n\times\mathbb{R}^n\to\mathbb{R}$ be a symmetric and positive semidefinite kernel, with feature map $\varphi:\mathcal{H}\to\mathbb{R}$.
Similarly to \cite{TraGil2006,MagSpi2023}, for each training observation $x^i\in\mathbb{R}^n$, we model the uncertainty set centered at $\varphi(x^i)\in\mathcal{H}$ as follows:
\begin{linenomath}
\begin{equation} \label{uncertainty_set_feature_space}
\mathcal{U}_{\mathcal{H}}\big(\varphi(x^i)\big):=\big\{z\in\mathcal{H} \lvert z=\varphi(x^i)+\widetilde{\delta}^i, \norm{\widetilde{\delta}^i}_{\mathcal{H}}\leq \widetilde{\varepsilon}_i\big\},
\end{equation}
\end{linenomath}
where the perturbation $\widetilde{\delta}^i$ belongs to $\mathcal{H}$ and its $\mathcal{H}$-norm is bounded by a constant $\widetilde{\varepsilon}_i\geq 0$. The value of $\widetilde{\varepsilon}_i$ depends on the radius $\varepsilon_i$ of the corresponding uncertainty set $\mathcal{U}_p(x^i)$ in the input space \eqref{uncertainty_set_input_space}. Closed-form expressions of $\widetilde{\varepsilon}_i$ for kernel functions $k(\cdot,\cdot)$ typically used in the ML literature are derived in \cite{MagSpi2023}. Additionally, a perturbation $\widetilde{\delta}^i$ in the feature space arises if and only if a perturbation $\delta^i$ occurs in the input space. Consequently, $\varepsilon_i=0$ implies $\widetilde{\varepsilon}_i=0$.

As in Section \ref{sec_rob_lin_multiclass}, we start by robustifying model \eqref{multiclass_nonlinear_TPMSVM} over the uncertainty set \eqref{uncertainty_set_feature_space}, obtaining the following optimization model:
\begin{linenomath}
\begin{equation} \label{robust_nonlinear_intractable}
\begin{aligned}
\min_{\widetilde{w}_c,{\widetilde{\theta}_c},\xi_c}  \quad & \frac{1}{2}\norm{\widetilde{w}_c}_{\mathcal{H}}^2+\frac{\nu_c}{m_{-c}} \sum_{i\in\mathcal{X}_{-c}} \max_{\norm{\widetilde{\delta}^i}_{\mathcal{H}}\leq \widetilde{\varepsilon}_i}  \big[\langle \widetilde{w}_c, \varphi(x^i)+\widetilde{\delta}^i \rangle_{\mathcal{H}} + {\widetilde{\theta}_c} \big]+\frac{\alpha_c}{m_c}\sum_{i \in \mathcal{X}_{c}} \xi_{c,i} \\
\text{s.t.} \quad & \langle \widetilde{w}_c, z \rangle_{\mathcal{H}} + {\widetilde{\theta}_c} \geq - \xi_{c,i} \qquad i \in \mathcal{X}_c, \forall z \in \mathcal{U}_{\mathcal{H}}\big(\varphi(x^i)\big) \\
 & \xi_{c,i}\geq 0 \qquad i \in \mathcal{X}_c.
\end{aligned}
\end{equation}
\end{linenomath}

Model \eqref{robust_nonlinear_intractable} is intractable in practice, due to the infinite possible choices of $z \in \mathcal{U}_{\mathcal{H}}\big(\varphi(x^i)\big)$ in the first set of constraints. Nevertheless, it can be reformulated in a tractable form, as stated in the following theorem.

\begin{theorem}
Let $\mathcal{U}_{\mathcal{H}}\big(\varphi(x^i)\big)$ be the uncertainty set as in \eqref{uncertainty_set_feature_space}. Model \eqref{robust_nonlinear_intractable} is equivalent to:
\begin{linenomath}
\begin{equation} \label{robust_nonlinear_tractable}
\begin{aligned}
\min_{\beta_c,{\widetilde{\theta}_c},\xi_c}  \quad & \frac{1}{2}\beta_{c}^\top K \beta_c+\frac{\nu_c}{m_{-c}} \sum_{i\in\mathcal{X}_{-c}} \bigg[\widetilde{\varepsilon}_i \sqrt{\beta_{c}^\top K \beta_{c}} + \sum_{j\in\mathcal{X}} \beta_{c,j} k(x^i,x^j)\bigg] + \nu_c{\widetilde{\theta}_c} +\frac{\alpha_c}{m_c}\sum_{i \in \mathcal{X}_{c}} \xi_{c,i} \\
\text{s.t.} \quad & {\widetilde{\theta}_c} -\widetilde{\varepsilon}_i \sqrt{\beta_{c}^\top K \beta_{c}} + \sum_{j\in\mathcal{X}} \beta_{c,j} k(x^i,x^j) \geq - \xi_{c,i} \qquad i \in \mathcal{X}_c \\
& \beta_{c,i}=-\frac{\nu_c}{m_{-c}} \qquad i\in\mathcal{X}_{-c} \\
 & \xi_{c,i}\geq 0 \qquad i \in \mathcal{X}_c.
\end{aligned}
\end{equation}
\end{linenomath}
\end{theorem}

\begin{proof}
Optimizing with respect to the worst-case realization across the uncertainty set \eqref{uncertainty_set_feature_space} yields the following equivalent reformulation of model \eqref{robust_nonlinear_intractable}:
\begin{linenomath}
\begin{equation} \label{model_proof_nonlinearr}
\begin{aligned}
\min_{\widetilde{w}_c,{\widetilde{\theta}_c},\xi_c}  \quad & \frac{1}{2}\norm{\widetilde{w}_c}_{\mathcal{H}}^2+\frac{\nu_c}{m_{-c}} \sum_{i\in\mathcal{X}_{-c}} \max_{\norm{\widetilde{\delta}^i}_{\mathcal{H}}\leq \widetilde{\varepsilon}_i}  \big[\langle \widetilde{w}_c, \varphi(x^i)+\widetilde{\delta}^i \rangle_{\mathcal{H}} + {\widetilde{\theta}_c} \big]+\frac{\alpha_c}{m_c}\sum_{i \in \mathcal{X}_{c}} \xi_{c,i} \\
\text{s.t.} \quad & \min_{\norm{\widetilde{\delta}^i}_{\mathcal{H}}\leq \widetilde{\varepsilon}_i}  \big[\langle \widetilde{w}_c, \varphi(x^i)+\widetilde{\delta}^i \rangle_{\mathcal{H}} + {\widetilde{\theta}_c} \big] \geq - \xi_{c,i} \qquad i \in \mathcal{X}_c \\
 & \xi_{c,i}\geq 0 \qquad i \in \mathcal{X}_c.
\end{aligned}
\end{equation}
\end{linenomath}
The maximization problem in the second term of the objective function in \eqref{model_proof_nonlinearr} can be rewritten as follows:
\begin{linenomath}
\begin{equation*}
\max_{\norm{\widetilde{\delta}^i}_{\mathcal{H}}\leq \widetilde{\varepsilon}_i}  \big[\langle \widetilde{w}_c, \varphi(x^i)+\widetilde{\delta}^i \rangle_{\mathcal{H}} + {\widetilde{\theta}_c} \big] = {\widetilde{\theta}_c} + \langle \widetilde{w}_c,\varphi(x^i) \rangle_{\mathcal{H}} + \max_{\norm{\widetilde{\delta}^i}_{\mathcal{H}}\leq \widetilde{\varepsilon}_i}  \big[\langle \widetilde{w}_c,\widetilde{\delta}^i \rangle_{\mathcal{H}}\big].
\end{equation*}
\end{linenomath}
The same argument holds for the minimization problem in the first set of constraints. By applying the Cauchy-Schwarz inequality in $\mathcal{H}$ and the structure of the uncertainty set \eqref{uncertainty_set_feature_space}, we obtain:
\begin{equation*}
\abs{\langle \widetilde{w}_c, \widetilde{\delta}^i \rangle_{\mathcal{H}} }  \leq \norm{\widetilde{w}_c}_{\mathcal{H}}  \norm{\widetilde{\delta}^i}_{\mathcal{H}}\leq \norm{\widetilde{w}_c}_{\mathcal{H}}  \widetilde{\varepsilon}_i \qquad \forall \widetilde{\delta}^i:\norm{\widetilde{\delta}^i}_{\mathcal{H}}\leq \widetilde{\varepsilon}_i.
\end{equation*}
Therefore, the optimal values of the maximization and minimization problems described above are respectively:
\begin{linenomath}
\begin{equation*}
\max_{\norm{\widetilde{\delta}^i}_{\mathcal{H}}\leq \widetilde{\varepsilon}_i}  \big[\langle \widetilde{w}_c, \varphi(x^i)+\widetilde{\delta}^i \rangle_{\mathcal{H}} + {\widetilde{\theta}_c} \big] = {\widetilde{\theta}_c} + \langle \widetilde{w}_c,\varphi(x^i) \rangle_{\mathcal{H}} + \widetilde{\varepsilon}_i \norm{\widetilde{w}_c}_{\mathcal{H}}
\end{equation*}
\end{linenomath}
and
\begin{linenomath}
\begin{equation*}
\min_{\norm{\widetilde{\delta}^i}_{\mathcal{H}}\leq \widetilde{\varepsilon}_i}  \big[\langle \widetilde{w}_c, \varphi(x^i)+\widetilde{\delta}^i \rangle_{\mathcal{H}} + {\widetilde{\theta}_c} \big] = {\widetilde{\theta}_c} + \langle \widetilde{w}_c,\varphi(x^i) \rangle_{\mathcal{H}} - \widetilde{\varepsilon}_i \norm{\widetilde{w}_c}_{\mathcal{H}}.
\end{equation*}
\end{linenomath}
Consequently, model \eqref{model_proof_nonlinearr} is equivalent to:
\begin{linenomath}
\begin{equation} \label{robust_nonlinear_intractable_proof}
\begin{aligned}
\min_{\widetilde{w}_c,{\widetilde{\theta}_c},\xi_c}  \quad & \frac{1}{2}\norm{\widetilde{w}_c}_{\mathcal{H}}^2+\frac{\nu_c}{m_{-c}} \sum_{i\in\mathcal{X}_{-c}} \bigg[\widetilde{\varepsilon}_i \norm{\widetilde{w}_c}_{\mathcal{H}} + \langle \widetilde{w}_c,\varphi(x^i) \rangle_{\mathcal{H}} \bigg] + \nu_c{\widetilde{\theta}_c}  +\frac{\alpha_c}{m_c}\sum_{i \in \mathcal{X}_{c}} \xi_{c,i} \\
\text{s.t.} \quad & {\widetilde{\theta}_c} - \widetilde{\varepsilon}_i \norm{\widetilde{w}_c}_{\mathcal{H}} + \langle \widetilde{w}_c,\varphi(x^i) \rangle_{\mathcal{H}}  \geq - \xi_{c,i} \qquad i \in \mathcal{X}_c \\
 & \xi_{c,i}\geq 0 \qquad i \in \mathcal{X}_c.
\end{aligned}
\end{equation}
\end{linenomath}
In the feature space $\mathcal{H}$, vector $\widetilde{w}_c$ can be decomposed as a linear combination of mapped input data by means of $\varphi(\cdot)$ (see \cite{CorVap1995}). Specifically, by considering equation \eqref{optimal_wc_bc_nonlinear}, it holds that:
\begin{linenomath}
\begin{equation*}
\widetilde{w}_{c} = \sum_{i\in\mathcal{X}_{c}} \lambda_{c,i} \varphi(x^i)-\frac{\nu_{c}}{m_{-c}}\sum_{i\in\mathcal{X}_{-c}} \varphi(x^i) = \sum_{i \in\mathcal{X}} \beta_{c,i} \varphi(x^i),
\end{equation*}
\end{linenomath}
where $\beta_{c}\in\mathbb{R}^m$ and $\beta_{c,i}=-\nu_c/m_{-c}$, for all $i\in\mathcal{X}_{-c}$. Therefore, the squared norm of $\widetilde{w}_{{c}}$ in the objective function of \eqref{robust_nonlinear_intractable_proof} corresponds to:
\begin{linenomath}
\begin{equation} \label{norm_wtilde_H}
\norm{\widetilde{w}_c}^2_{\mathcal{H}}=\langle \widetilde{w}_c,\widetilde{w}_c \rangle_{\mathcal{H}} = \sum_{i,j\in\mathcal{X}} \beta_{c,i}k(x^i,x^j)\beta_{c,j} = \beta_{c}^\top K \beta_{c},
\end{equation}
\end{linenomath}
where $K$ is the Gram matrix of the entire training set, i.e. $K_{ij}=K_{ji}=k(x^i,x^j)$, with $i,j\in\mathcal{X}$. For all $i\in\mathcal{X}_c$, the dot product $\langle \widetilde{w}_c, \varphi(x^i) \rangle_{\mathcal{H}}$ in the objective function and in the first set of constraints of \eqref{robust_nonlinear_intractable_proof} can be rewritten as:
\begin{equation} \label{dot_product_proof}
\begin{linenomath}
\langle \widetilde{w}_c, \varphi(x^i) \rangle_{\mathcal{H}} = \sum_{j\in\mathcal{X}} \beta_{c,j} k(x^i,x^j).
\end{linenomath}
\end{equation}
Thus, model \eqref{robust_nonlinear_intractable_proof} is equivalent to:
\begin{linenomath}
\begin{equation*}
\begin{aligned}
\min_{\beta_c,{\widetilde{\theta}_c},\xi_c}  \quad & \frac{1}{2}\beta_{c}^\top K \beta_c+\frac{\nu_c}{m_{-c}} \sum_{i\in\mathcal{X}_{-c}} \bigg[\widetilde{\varepsilon}_i \sqrt{\beta_{c}^\top K \beta_{c}} + \sum_{j\in\mathcal{X}} \beta_{c,j} k(x^i,x^j)\bigg] + \nu_c{\widetilde{\theta}_c}  +\frac{\alpha_c}{m_c}\sum_{i \in \mathcal{X}_{c}} \xi_{c,i} \\
\text{s.t.} \quad & {\widetilde{\theta}_c} -\widetilde{\varepsilon}_i \sqrt{\beta_{c}^\top K \beta_{c}} + \sum_{j\in\mathcal{X}} \beta_{c,j} k(x^i,x^j) \geq - \xi_{c,i} \qquad i \in \mathcal{X}_c \\
& \beta_{c,i}=-\frac{\nu_c}{m_{-c}} \qquad i\in\mathcal{X}_{-c} \\
 & \xi_{c,i}\geq 0 \qquad i \in \mathcal{X}_c.
\end{aligned}
\end{equation*}
\end{linenomath}
This concludes the proof.
\end{proof}

Model \eqref{robust_nonlinear_tractable} can be reformulated as the following SOCP model by introducing variables $t_c$, $u_c$, $v_c \in \mathbb{R}$:
\begin{linenomath}
\begin{equation} \label{robust_nonlinear_SOCP}
\begin{aligned}
\min_{\beta_c,{\widetilde{\theta}_c},\xi_c,t_c,u_c,v_c}  \quad & \frac{1}{2}(u_c-v_c)+\frac{\nu_c}{m_{-c}} \sum_{i\in\mathcal{X}_{-c}} \bigg[\widetilde{\varepsilon}_i t_c + \sum_{j\in\mathcal{X}} \beta_{c,j} k(x^i,x^j)\bigg] + \nu_c{\widetilde{\theta}_c}  +\frac{\alpha_c}{m_c}\sum_{i \in \mathcal{X}_{c}} \xi_{c,i} \\
\text{s.t.} \quad & {\widetilde{\theta}_c} -\widetilde{\varepsilon}_i t_c + \sum_{j\in\mathcal{X}} \beta_{c,j} k(x^i,x^j) \geq - \xi_{c,i} \qquad i \in \mathcal{X}_c \\
& \beta_{c,i}=-\frac{\nu_c}{m_{-c}} \qquad i\in\mathcal{X}_{-c} \\
& t_c \geq \sqrt{\beta_{c}^\top K \beta_c}\\
& u_c+v_c = 1\\
& u_c \geq \sqrt{t_c^2+v_c^2}\\
 & \xi_{c,i}\geq 0 \qquad i \in \mathcal{X}_c.
\end{aligned}
\end{equation}
\end{linenomath}
Once $\beta_c$ and ${\widetilde{\theta}_c}$ are determined as solutions of \eqref{robust_nonlinear_SOCP}, the classification of a new observation $x\in\mathbb{R}^n$ is performed using the decision functions \eqref{dec_func_multiclass_nonlinear_argmin}-\eqref{dec_func_multiclass_nonlinear_argmax}. In these functions, the $\mathcal{H}$-norm of $\widetilde{w}_c$ is computed as in \eqref{norm_wtilde_H}, while the dot product $\langle \widetilde{w}_{c},\varphi(x) \rangle_{\mathcal{H}}$ is calculated similarly to \eqref{dot_product_proof}, i.e.:
\begin{linenomath}
\begin{equation*}
\langle \widetilde{w}_{c},\varphi(x) \rangle_{\mathcal{H}} = \sum_{i\in\mathcal{X}} \beta_{c,i} \langle \varphi(x^i), \varphi(x) \rangle_{\mathcal{H}} = \sum_{i\in\mathcal{X}} \beta_{c,i} k(x^i,x).
\end{equation*}
\end{linenomath}

Using the same complexity analysis outlined in Section \ref{sec_rob_lin_multiclass} for robust SOCP models with linear classifiers, we conclude that a valid upper bound on the total computational complexity of problem \eqref{robust_nonlinear_SOCP} is $ C \cdot \mathcal{O}\left((2m + 4)^{7/2} \log(1/\eta)\right) $.

\section{Experimental results} \label{sec_numerical_results}
In this section, we evaluate the performance of the deterministic TPMSVM models presented in Section \ref{sec_multiclass_deterministic} and their robust counterparts derived in Section \ref{sec_multiclass_robust} on a selection of five public-domain multiclass datasets. All models were implemented in MATLAB (v. 2024b) and solved using CVX (see \cite{GraBoy2008CVX,GraBoy2014CVX}) with MOSEK as solver (v. 9.1.9, see \cite{mosek}). All computational experiments were run on a MacBookPro17.1 with a chip Apple M1 of 8 cores and 16 GB of RAM memory. The codes were parallelized across the machine's cores using the MATLAB Parallel Computing Toolbox.

This section is organized as follows. In Section \ref{sec_results_dataset_setting}, we provide a description of the benchmark datasets and the experimental setting. The performance of the deterministic approaches are presented in Section \ref{sec_results_deterministic}, while the results for the robust models are discussed in Section \ref{sec_results_robust}.

\subsection{Datasets and experimental setting} \label{sec_results_dataset_setting}
We considered five public-domain real-world multiclass datasets from the \emph{UCI Machine Learning} repository (UCI, \cite{KelLonNot2023}) and from \emph{Open Data Canada} (ODC, \cite{Canada}). Relevant details of the datasets are summarized in Table \ref{tab_stat_dataset}.

\begin{table}[h!]
\centering
\resizebox{\textwidth}{!}{
\begin{tabular}{llllll}\toprule
Dataset & Source & Subject area & Observations $(m)$ & Features $(n)$ & Classes $(C)$
\\
\hline
Iris & UCI & Biology & 150 & 4 & 3
\\
Wine & UCI & Physics and Chemistry & 178 & 13 & 3
\\
Glass Identification & UCI & Physics and Chemistry & {214} & {9} & {6}
\\
Fuel Consumption Ratings & ODC & Transport & 374 & 7 & 3
\\
Car Evaluation & {UCI} & {Other} & {1728} & {6} & {4}
\\
\bottomrule
\end{tabular}}
\caption{Summary statistics of considered datasets.} \label{tab_stat_dataset}
\end{table}

To assess the performance of the proposed methodology, each dataset was randomly divided into training and testing sets, with a proportion 75\%-25\% of the total number of observations. The partition follows the \emph{proportional random sampling} strategy (see \cite{ChenTseYu2001}), thereby preserving the original class distribution both in the training and in the testing set. To avoid imbalances among the orders of magnitude of the features, each dataset was linearly scaled into the $n$-dimensional hypercube $[0,1]^n$. For kernel-based classifiers, we considered five polynomial kernels (homogeneous quadratic and cubic; inhomogeneous linear, quadratic and cubic) as well as the Gaussian kernel (see Table \ref{tab_kernels}).

As far as it concerns hyperparameters, for simplicity we set $\nu_c=\nu$ and $\alpha_c=\alpha$ for all $c=1,\ldots,C$. A grid-search procedure was then conducted to tune their values. Specifically, $\alpha$ was selected from the set $\{2^j|j=-6,-5,\ldots,5,6\}$, whereas the value of $\nu/\alpha$ was chosen from the set $\{0.1,0.3,0.5,0.7,0.9\}$. Additionally, the parameters $\gamma$ for the inhomogeneous polynomial kernels and $\sigma$ for the Gaussian kernel were selected from $\{2^j|j=-4,-3,\ldots,3,4\}$.

The best configuration of parameters was identified as the one maximizing the accuracy on the training set. The performance of the model was then evaluated on the testing set and the corresponding accuracy was computed. To ensure stability on the results, for each hold-out 75\%-25\% the computational experiments were performed over 50 random combinations of training and testing sets, and the results were finally averaged.

\subsection{Results for the deterministic multiclass TPMSVM models}  \label{sec_results_deterministic}
Tables \ref{tab_res_determ_mindist}-\ref{tab_res_determ_maxdist} report the results of deterministic models \eqref{multiclass_linear_TPMSVM} and \eqref{multiclass_nonlinear_TPMSVM} in terms of percentage average accuracy and standard deviation. The CPU time for training each model is outlined below the results.

\begin{table}[h!]
    \begin{subtable}[h]{\textwidth}
\centering
\resizebox{\textwidth}{!}{
\begin{tabular}{ll|l| *{6}l}\toprule
Dataset & & & \multicolumn{6}{c}{Kernel}\\
 & & Linear & Hom. quadratic & Hom. cubic & Inhom. linear & Inhom. quadratic & Inhom. cubic & Gaussian \\ \toprule
\multirow{2}{*}{Iris} & Accuracy & {$\mathbf{{92.81 \pm 4.29}}$} & {$78.54 \pm 7.67$} & {$85.68 \pm 5.58$} & {$\underline{92.38 \pm 3.89}$} & {$91.35 \pm 5.09$} & {$88.00 \pm 5.46$} & {$87.78 \pm 5.78$}
\\
& CPU time (s) &1.61&	1.19	&1.22	&10.61	&10.20	&10.32&	10.63
\\
\cmidrule{1-9}
\multirow{2}{*}{Wine} & Accuracy & {$97.00 \pm 2.08$} & {$96.59 \pm 2.98$} & {$95.23 \pm 3.22$} & {$96.64 \pm 2.57$} & {$\underline{\mathbf{97.41 \pm 2.43}}$} & {$95.91 \pm 3.47$} & {$39.45 \pm 1.10$}
\\
& CPU time (s) &1.50	&1.22&	1.11&	10.04&	9.98	&10.00	&10.20
\\
\cmidrule{1-9}
\multirow{2}{*}{{Glass}} & {Accuracy} & {$38.49 \pm 7.17$} & {$41.81 \pm 5.75$} &	{$44.26 \pm 5.22$} &	{$36.94 \pm 6.97$} &	{$36.34 \pm 6.95$} &	{$44.38 \pm 6.39$} &	{$\underline{\mathbf{61.17 \pm 5.23}}$}
\\
& {CPU time (s)} & 3.07	&2.35	&2.25	&19.77	&19.61&	19.87&	20.51
\\
\cmidrule{1-9}
\multirow{2}{*}{Fuel} & Accuracy & {$52.73 \pm 5.35$} & {$50.47 \pm 4.29$} & {$51.68 \pm 5.78$} & {$55.66 \pm 4.67$} & {$52.62 \pm 5.31$} & {$53.42 \pm 5.48$} & {$\underline{\mathbf{70.02 \pm 4.84}}$}
\\
& CPU time (s) & 1.68	&1.24	&1.14	&9.67	&9.92	&10.84&	11.17
\\
\cmidrule{1-9}
\multirow{2}{*}{{Car}} & {Accuracy} & {$73.15 \pm 2.00$} & {$77.71 \pm 2.15$} &	{$79.07 \pm 1.88$} &	{$76.62 \pm 1.99$} &	{$78.14 \pm 1.85$} &	{$\underline{\mathbf{79.12 \pm 2.20}}$} &	{$69.96 \pm 0.10$}
\\
& {CPU time (s)} & 2.06&	2.17	&3.05	&20.42	&20.23	&28.85	&46.70
\\
\bottomrule
\end{tabular}
}
\caption{Deterministic models - argmin decision functions \eqref{dec_func_multiclass_linear_argmin} and \eqref{dec_func_multiclass_nonlinear_argmin}.}
\label{tab_res_determ_mindist}
\end{subtable}
\vfill \vspace*{0.25cm}
    \begin{subtable}[h]{\textwidth}
        \centering
        \resizebox{\textwidth}{!}{
\begin{tabular}{ll|l| *{6}l}\toprule
Dataset & & & \multicolumn{6}{c}{Kernel}\\
 & & Linear & Hom. quadratic & Hom. cubic & Inhom. linear & Inhom. quadratic & Inhom. cubic & Gaussian \\ \toprule
\multirow{2}{*}{Iris} & Accuracy & {$70.22 \pm 2.47$} & {$94.59 \pm 4.01$} & {$92.65 \pm 6.23$} & {$76.92 \pm 5.60$} & {$\underline{\mathbf{95.30 \pm 3.53}}$} & {$94.49 \pm 3.54$} & {$87.62 \pm 4.64$}
\\
& CPU time (s) & 1.59	&1.15	&1.15	&10.07	&10.15	&10.29	&10.56
\\
\cmidrule{1-9}
\multirow{2}{*}{Wine} & Accuracy & {$\mathbf{96.55 \pm 3.44}$} & {$96.36 \pm 3.15$} & {$95.55 \pm 3.47$} & {$\underline{96.41 \pm 2.39}$} & {$95.95 \pm 2.53$} & {$95.82 \pm 2.36$} & {$39.45 \pm 1.10$}
\\
& CPU time (s) & 1.57&	1.13&	1.13&	10.14&	10.28	&10.26	&10.53
\\
\cmidrule{1-9}
\multirow{2}{*}{{Glass}} & {Accuracy} & {$46.42 \pm 6.40$} & {$45.13 \pm 7.11$} & {$47.28 \pm 6.99$} &	{$41.17 \pm 5.29$} &	{$45.74 \pm 7.09$} &	{$49.02 \pm 6.65$} &	{$\underline{\mathbf{61.58 \pm 5.21}}$}
\\
& {CPU time (s)} & 3.05	&2.29	&2.26	&19.69&	20.16&	19.51&	20.16
\\
\cmidrule{1-9}
\multirow{2}{*}{Fuel} & Accuracy & {$54.92\pm 5.90$} & {$56.86 \pm 4.90$} & {$56.65 \pm 6.08$} & {$57.57 \pm 5.20$} & {$57.72 \pm 5.27$} & {$58.37 \pm 4.23$} & {$\underline{\mathbf{66.60 \pm 4.53}}$}
\\
& CPU time (s) &1.69&	1.23	&1.13	&10.03	&10.18	&10.81	&11.28
\\
\cmidrule{1-9}
\multirow{2}{*}{{Car}} & {Accuracy} & {$75.49 \pm 1.64$} & {$81.50 \pm 1.71$} & {$81.82 \pm 1.38$} & {$77.23 \pm 1.44$} & {$82.25 \pm 1.52$} & {$\mathbf{\underline{84.07 \pm 1.92}}$} & {$69.96 \pm 0.10$}
\\
& {CPU time (s)} & 2.06	&2.20	&3.13	&19.16&	19.58&	28.78	&47.11
\\
\bottomrule
\end{tabular}
}
\caption{Deterministic models - argmax decision functions \eqref{dec_func_multiclass_linear_argmax} and \eqref{dec_func_multiclass_nonlinear_argmax}.}
\label{tab_res_determ_maxdist}
    \end{subtable}

\caption{Detailed percentage results of average accuracy and standard deviation over 50 runs of the deterministic models \eqref{multiclass_linear_TPMSVM} and \eqref{multiclass_nonlinear_TPMSVM}. Classification is performed according to the argmin and argmax decision functions, respectively in Table \ref{tab_res_determ_mindist} and \ref{tab_res_determ_maxdist}. The best result for the kernelized model is underlined. Overall, the best result is in bold.} \label{tab_res_determ}
\end{table}

First of all, by comparing the third column of Tables \ref{tab_res_determ_mindist}-\ref{tab_res_determ_maxdist} with the remaining ones, we observe that the best kernel-induced classifier consistently outperforms the separating hyperplane obtained from model \eqref{multiclass_linear_TPMSVM} across the majority of the tasks. Indeed, nonlinear classifiers provide better accuracy regardless of the decision function used, showing the strength of kernelized approaches. Even in the few cases where the linear classifier achieves slightly better accuracy (see the \emph{Iris} dataset with the argmin decision function and the \emph{Wine} dataset with the argmax decision function), the accuracy differences between the results are negligible ($0.43\%$ and $0.14\%$, respectively).

Secondly, the choice of the decision function plays an important role on the performance of the models, particularly for the \emph{Iris} dataset. In the linear case, the accuracy drops from $92.81\%$ (see Table \ref{tab_res_determ_mindist}) using the argmin decision function to $70.22\%$ (see Table \ref{tab_res_determ_maxdist}) with the argmax decision function. This highlights that the argmin decision function \eqref{dec_func_multiclass_linear_argmin} is better suited for solving linear classification tasks within this dataset. Conversely, in the nonlinear setting, the argmax decision function \eqref{dec_func_multiclass_nonlinear_argmax} yields the best results for certain kernels: accuracy improves from $78.54\%$ to $94.59\%$ with the homogeneous quadratic kernel ($+16.05\%$), and from $85.68\%$ to $92.65\%$ with the homogeneous cubic kernel ($+6.97\%)$. For the \emph{Wine} dataset, both decision functions lead to comparable performance across all kernels, suggesting flexibility in choice. Finally, for the remaining datasets (\emph{Glass}, \emph{Fuel}, and \emph{Car}) the argmax decision function generally delivers superior accuracy.

In terms of computational efficiency, kernel-induced models, and particularly the ones with the Gaussian kernel, require significantly more CPU time compared to linear classifiers. Therefore, we can conclude that there exists a trade-off between accuracy and performing speed. The final user must balance these factors based on their priority: faster results with good accuracy (linear classifiers) or longer runtimes with better predictive performance (nonlinear classifiers).

\subsection{Results for the robust multiclass TPMSVM models} \label{sec_results_robust}
To evaluate the performance of the robust approaches, we assume a constant radius for the uncertainty set \eqref{uncertainty_set_input_space} across all observations, i.e. $\varepsilon_i=\varepsilon$ for all $i\in\mathcal{X}$. We consider three increasing levels of perturbation ($\varepsilon=10^{-3}$, $10^{-2}$, $10^{-1}$) and three different $\ell_p$-norms ($p=1$, $2$, $\infty$). 

The results of the numerical experiments for the robust linear classification tasks are presented in Tables \ref{tab_res_linear_robust_mindist}-\ref{tab_res_linear_robust_maxdist}.

\begin{table}[h!]
    \begin{subtable}[h]{\textwidth}
\centering
\resizebox{\textwidth}{!}{
\begin{tabular}{ll| *{3}l| *{3}l| *{3}l}\toprule
Dataset & $\ell_p$-norm & \multicolumn{3}{c}{$p=1$} & \multicolumn{3}{c}{$p=2$} & \multicolumn{3}{c}{$p=\infty$}\\
& $\varepsilon$ & $10^{-3}$ & $10^{-2}$ & $10^{-1}$ & $10^{-3}$ & $10^{-2}$ & $10^{-1}$ & $10^{-3}$ & $10^{-2}$ & $10^{-1}$
\\ \toprule
\multirow{2}{*}{Iris} & Accuracy & {$92.70 \pm 3.79$} &	{$93.89 \pm 3.93$} &	{$\underline{\mathbf{95.35 \pm 3.32}}$} &	{$\underline{93.19 \pm 3.46}$} &	{$93.14 \pm 3.38$} & {$92.49 \pm 4.24$}	& {$91.95 \pm 3.72$} &	 {$\underline{92.38 \pm 4.29}$} & {$78.59 \pm 5.93$}
\\
& CPU time (s) & {1.54} & {1.46}  & {1.45} & {1.41} & {1.40} & {1.40} & {1.46} & {1.46} & {1.45}
\\
\cmidrule{1-11}
\multirow{2}{*}{Wine} & Accuracy & {$97.36 \pm 2.52$} & {$\underline{\mathbf{97.59 \pm 1.92}}$} & {$96.64 \pm 2.84$} & {$97.14 \pm 2.19$} & {$\underline{97.23 \pm 2.12}$} & {$96.59 \pm 2.12$} & {$96.64 \pm 2.40$} & {$\underline{96.91 \pm 2.09}$} & {$93.86 \pm 3.16$}
\\
& CPU time (s) & {1.44} & {1.46} & {1.46} & {1.42} & {1.42} & {1.46} & {1.48} & {1.46} & {1.45}
\\
\cmidrule{1-11}
\multirow{2}{*}{{Glass}} & {Accuracy} & {$\underline{39.32 \pm 6.06}$} & {$38.15 \pm 6.62$}  & {$34.00 \pm 7.14$} & {$\underline{38.64 \pm 7.32}$} & {$37.85 \pm 6.92$} & {$35.25 \pm 6.82$} & {$37.92 \pm 6.68$} & {$37.96 \pm 7.18$} & {$\underline{\mathbf{39.92 \pm 6.06}}$}
\\
& {CPU time (s)} & {3.04} & {2.91} & {2.85} & {2.92} & {2.83} & {2.85} & {2.96} & {2.93} & {2.97}
\\
\cmidrule{1-11}
\multirow{2}{*}{Fuel} & Accuracy & {$51.66 \pm 5.30$} & {$52.65 \pm 4.95$} & {$\underline{52.90 \pm 4.62}$} & {$51.48 \pm 4.48$} & {$\underline{\mathbf{55.40 \pm 5.59}}$} &	{$53.94 \pm 4.55$} & {$51.48 \pm 3.77$} & {$\underline{53.05 \pm 4.08}$} & {$51.01 \pm 3.71$}
\\
& CPU time (s) & {1.51} & {1.44} & {1.47} & {1.41} & {1.41} & {1.39} & {1.47} & {1.44} & {1.44}
\\
\cmidrule{1-11}
\multirow{2}{*}{{Car}} & {Accuracy} & {$71.99 \pm 1.96$} & {$\underline{72.00 \pm 2.19}$} & {$71.73 \pm 1.97$} & {$\underline{\mathbf{72.55 \pm 1.88}}$} & {$72.50 \pm 2.18$} & {$70.25 \pm 2.16$} & {$\underline{72.12 \pm 2.00}$} & {$70.67 \pm 2.40$} & {$58.74 \pm 2.48$}
\\
& {CPU time (s)} & {1.97} & {1.96} & {2.01} & {1.94} & {1.94} & {1.94} & {2.02} & {2.01} & {2.02}
\\
\bottomrule
\end{tabular}
}
\caption{Robust model - linear classifier - argmin decision function \eqref{dec_func_multiclass_linear_argmin}.}
\label{tab_res_linear_robust_mindist}
\end{subtable}
\vfill \vspace*{0.25cm}
    \begin{subtable}[h]{\textwidth}
\centering
\resizebox{\textwidth}{!}{
\begin{tabular}{ll| *{3}l| *{3}l| *{3}l}\toprule
Dataset & $\ell_p$-norm & \multicolumn{3}{c}{$p=1$} & \multicolumn{3}{c}{$p=2$} & \multicolumn{3}{c}{$p=\infty$}\\
& $\varepsilon$ & $10^{-3}$ & $10^{-2}$ & $10^{-1}$ & $10^{-3}$ & $10^{-2}$ & $10^{-1}$ & $10^{-3}$ & $10^{-2}$ & $10^{-1}$
\\ \toprule
\multirow{2}{*}{Iris} & Accuracy & {$69.57 \pm 2.88$} &	 {$\underline{70.11 \pm 3.01}$} & {$69.24 \pm 2.72$} &	{$\underline{70.65 \pm 2.94}$} &	{$70.38 \pm 3.78$} &	{$69.78 \pm 2.98$} &	{$70.54 \pm 3.24$} &	{$69.57 \pm 2.83$} &	{$\underline{\mathbf{84.38 \pm 4.83}}$}
\\
& CPU time (s) & {1.57} & {1.47} & {1.52} & {1.43} & {1.41} & {1.41} & 1.48 & {1.47} & {1.49}
\\
\cmidrule{1-11}
\multirow{2}{*}{Wine} & Accuracy & {$\underline{97.00 \pm 2.27}$} & {$96.82 \pm 2.05$} & {$96.68 \pm 2.80$} & {$97.09 \pm 2.06$} & {$96.36 \pm 2.30$} & {$\underline{\mathbf{97.14 \pm 2.29}}$} & {$96.32 \pm 2.29$} & {$\underline{97.00 \pm 2.45}$} & {$96.27 \pm 2.55$}
\\
& CPU time (s) & {1.48} & {1.54} & {1.46} & {1.42} & {1.42} & {1.40} & {1.50} & {1.49} & {1.51}
\\
\cmidrule{1-11}
\multirow{2}{*}{{Glass}} & {Accuracy} & {$45.13 \pm 6.13$} & {$\underline{\mathbf{47.85 \pm 7.05}}$}  & {$45.47 \pm 5.88$} & {$\underline{46.83 \pm 7.39}$} & {$45.89 \pm 7.65$} & {$44.75 \pm 5.30$} & {$46.53 \pm 7.78$} & {$\underline{47.09 \pm 6.22}$} & {$42.87 \pm 4.36$}
\\
& {CPU time (s)} & {2.92} & {2.89} & {2.95} & {2.82} & {2.84} & {2.82} & {2.93} & {2.95} & {2.94}
\\
\cmidrule{1-11}
\multirow{2}{*}{Fuel} & Accuracy & {$56.34\pm 6.04$} & {$53.57 \pm 4.15$} & {$\underline{\mathbf{60.65 \pm 3.48}}$}	& {$55.46 \pm 6.05$} & {$57.53 \pm 4.13$} & {$\underline{60.26 \pm 3.22}$} & {$54.30 \pm 5.04$} & {$\underline{57.12 \pm 5.58}$} & {$54.00 \pm 5.14$}
\\
& CPU time (s) & {1.49} & {1.46} & {1.43} & {1.39} & {1.39} & {1.40} & {1.43} & {1.43} & {1.44}
\\
\cmidrule{1-11}
\multirow{2}{*}{{Car}} & {Accuracy} & {$75.11 \pm 1.79$} & {$75.53 \pm 1.53$} & {$\underline{76.22 \pm 1.33}$} & {$75.71 \pm 1.35$} & {$75.81 \pm 1.36$} & {$\underline{77.72 \pm 1.75}$} & {$75.48 \pm 1.74$} & {$75.66 \pm 1.82$} & {$\underline{\mathbf{79.42 \pm 1.45}}$}
\\
& {CPU time (s)} & {2.04} & {2.05} & {2.04} & {1.94} & {1.97} & {1.97} & {2.05} & {2.07} & {2.05}
\\
\bottomrule
\end{tabular}
}
\caption{Robust model - linear classifier - argmax decision function \eqref{dec_func_multiclass_linear_argmax}.}
\label{tab_res_linear_robust_maxdist}
    \end{subtable}

\caption{Detailed percentage results of average accuracy and standard deviation over 50 runs of the robust model \eqref{robust_linear_tractable} with linear classifier. Classification is performed according to the argmin and argmax decision functions, respectively in Table \ref{tab_res_linear_robust_mindist} and \ref{tab_res_linear_robust_maxdist}. The best result for each $\ell_p$-norm is underlined. Overall, the best result is in bold.} \label{tab_res_linear_robust}
\end{table}

A comparison between these results and those in the third column of Tables \ref{tab_res_determ_mindist}-\ref{tab_res_determ_maxdist} shows that in nine out of ten cases the robust model outperforms the corresponding deterministic formulation. This is particularly evident when using the argmax decision function \eqref{dec_func_multiclass_linear_argmax}. For instance, considering the \emph{Iris} dataset with $p=\infty$ and $\varepsilon=10^{-1}$, the robust linear model achieves an accuracy of $84.38\%$ (see Table \ref{tab_res_linear_robust_maxdist}), compared to $70.22\%$ for the deterministic counterpart (see Table \ref{tab_res_determ_maxdist}). A significant improvement is observed for the \emph{Car} dataset too ($79.42\%$ vs $75.49\%$).

Regarding the choice of the $\ell_p$-norm used to define the uncertainty set, the results show that no single norm consistently yields the best performance across all tasks. Specifically, the $\ell_1$-norm leads to the highest accuracy in four cases, while the $\ell_2$- and $\ell_\infty$-norms are the best in three cases each. This variability suggests that the effectiveness of a particular norm is problem-dependent and may be influenced by the structure of the data.

In addition to the choice of norm, the perturbation level $\varepsilon$ also plays a crucial role in model performance. As $\varepsilon$ increases from $10^{-3}$ to $10^{-1}$, accuracy varies depending on the dataset and the choice of $\ell_p$-norm. The general trend indicates that increasing the perturbation level can be beneficial for the model. In fact, in only one out of ten cases the best accuracy is achieved at the lowest perturbation level ($\varepsilon = 10^{-3}$, see the \emph{Car} dataset in Table \ref{tab_res_linear_robust_mindist}), while in six out of ten cases the best results are obtained with the highest level of perturbation ($\varepsilon = 10^{-1}$). This suggests that introducing a controlled level of uncertainty can improve predictive performance. Particularly, a significant improvement is observed for the \emph{Iris} dataset with the argmax decision function and $\ell_\infty$-norm (see Table \ref{tab_res_linear_robust_maxdist}), where increasing $\varepsilon$ leads to a substantial gain in accuracy: from $70.54\%$ at $\varepsilon = 10^{-3}$ to $84.38\%$ at $\varepsilon = 10^{-1}$. On the other hand, overly conservative settings may deteriorate performance in some specific cases, as observed for the \emph{Car} dataset (see Table \ref{tab_res_linear_robust_mindist}), where too large perturbations lead to a decline in accuracy. All these findings highlight the importance of carefully selecting the perturbation level and an appropriate norm to balance robustness and predictive performance, potentially validated through empirical tuning.

In the case of the robust multiclass model \eqref{robust_nonlinear_SOCP} with nonlinear classifiers, we selected as kernel function the one that achieved the highest accuracy in the deterministic setting (see Tables \ref{tab_res_determ_mindist}-\ref{tab_res_determ_maxdist}). The results of the numerical experiments are reported in Tables \ref{tab_res_nonlinear_robust_mindist}-\ref{tab_res_nonlinear_robust_maxdist}.

\begin{table}[h!]
  \begin{subtable}[h]{\textwidth}
\centering
\resizebox{\textwidth}{!}{
\begin{tabular}{l| l| l| *{3}l| *{3}l| *{3}l}\toprule
Dataset & Kernel & $\ell_p$-norm & \multicolumn{3}{c}{$p=1$} & \multicolumn{3}{c}{$p=2$} & \multicolumn{3}{c}{$p=\infty$}\\
& & $\varepsilon$ & $10^{-3}$ & $10^{-2}$ & $10^{-1}$ & $10^{-3}$ & $10^{-2}$ & $10^{-1}$ & $10^{-3}$ & $10^{-2}$ & $10^{-1}$
\\ \toprule
\multirow{2}{*}{Iris} & \multirow{2}{*}{Inhom. linear} & Accuracy & {$\underline{92.97 \pm 3.98}$} & {$92.86 \pm 4.15$} & {$90.32 \pm 6.93$} & {$\underline{\mathbf{93.30 \pm 3.79}}$} & {$92.27 \pm 3.54$} & {$89.24 \pm 7.96$} & {$91.46 \pm 3.76$} & {$\underline{92.76\pm3.64}$} & $86.76 \pm 10.98$
\\
& & CPU time (s) & {13.09} & {12.70} & {13.05} & {12.90} & {13.51} & {13.08} & {13.09} & {13.04} & {12.83}
\\
\cmidrule{1-12}
\multirow{2}{*}{Wine} & \multirow{2}{*}{Inhom. {quadratic}} & Accuracy &  {$96.14 \pm 2.61$} & {$\underline{96.86 \pm 2.94}$} & {$96.05 \pm 2.75$} & {$96.00 \pm 2.89$} & {$\underline{\mathbf{96.91 \pm 2.86}}$} & {$96.45 \pm 2.39$} & {$96.45 \pm 2.64$} & {$\underline{96.64 \pm 2.84}$} & {$7.09 \pm 3.14$}
\\
& & CPU time (s) & {101.51} & {101.51} & {99.93} & {100.82} & {97.19} & {99.03} & {101.68} & {106.09} & {99.37}
\\
\cmidrule{1-12}
\multirow{2}{*}{{Glass}} & \multirow{2}{*}{{Gaussian}} & {Accuracy} & {$\underline{\mathbf{39.86 \pm 8.22}}$} & {$37.26 \pm 3.46$} & {$36.79 \pm 4.16$} & {$36.56 \pm 3.34$} & {$\underline{37.26 \pm 3.46}$} & {$36.79 \pm 4.16$} & {$37.26 \pm 3.87$} & {$\underline{37.74 \pm 3.51}$} & {$\underline{37.74 \pm 3.34}$}
\\
& & {CPU time (s)} & {2598.09} & {2810.71} & {2869.73} & {2772.89} & {2893.72} & {2871.30} & {2701.93} & {2717.65} & {2877.96}
\\
\cmidrule{1-12}
\multirow{2}{*}{Fuel} & \multirow{2}{*}{Gaussian} & Accuracy & {$\underline{\mathbf{62.99 \pm 5.69}}$} & {$62.75 \pm 5.96$} & {$40.26 \pm 2.63$} & {$60.08 \pm 3.87$} & {$\underline{60.35 \pm 2.90}$} & {$39.52\pm 1.70$} & {$\underline{61.56\pm 2.68}$} & {$59.01 \pm 3.65$} & {$37.63 \pm 0.81$}
\\
& & CPU time (s) & {5341.74} & {5163.59} & {4985.68} & {4885.93} & {5186.56} & {4843.79} & {5083.51} & {5019.50} & {5197.60}
\\
\cmidrule{1-12}
\multirow{2}{*}{{Car}} & \multirow{2}{*}{{Inhom. cubic}} & {Accuracy} & {$76.24 \pm 1.95$} & {$\underline{\mathbf{76.59 \pm 2.18}}$} & {$67.29 \pm 2.20$} & {$\underline{75.72 \pm 2.27}$} & {$75.33 \pm 2.30$} & {$60.45\pm 2.60$} & {$\underline{75.74\pm 2.01}$} & {$74.61 \pm 2.26$} & {$42.17 \pm 5.02$}
\\
& & {CPU time (s)} & {2978.02} & {2979.47} & {2984.12} & {2954.32} & {2902.26} & {2909.32} & {3018.96} & {3033.63} & {2950.35}
\\
\bottomrule
\end{tabular}
}
\caption{Robust model - nonlinear classifier - argmin decision function \eqref{dec_func_multiclass_nonlinear_argmin}.}
\label{tab_res_nonlinear_robust_mindist}
\end{subtable}
\vfill \vspace*{0.25cm}
    \begin{subtable}[h]{\textwidth}
\centering
\resizebox{\textwidth}{!}{
\begin{tabular}{l| l| l| *{3}l| *{3}l| *{3}l}\toprule
Dataset & Kernel & $\ell_p$-norm & \multicolumn{3}{c}{$p=1$} & \multicolumn{3}{c}{$p=2$} & \multicolumn{3}{c}{$p=\infty$}\\
& & $\varepsilon$ & $10^{-3}$ & $10^{-2}$ & $10^{-1}$ & $10^{-3}$ & $10^{-2}$ & $10^{-1}$ & $10^{-3}$ & $10^{-2}$ & $10^{-1}$
\\ \toprule
\multirow{2}{*}{Iris} & \multirow{2}{*}{Inhom. {quadratic}} & Accuracy & {$93.78 \pm 3.75$} & {$\underline{95.46 \pm 3.84}$} & {$93.51 \pm 3.98$} & {$94.27 \pm 3.81$} & {$94.97 \pm 3.86$} & {$\underline{95.14 \pm 3.94}$} & {$94.43 \pm 3.47$} & {$\underline{\mathbf{95.46\pm 3.34}}$} & {$94.32 \pm 3.83$}
\\
& & CPU time (s) & {13.27}& {	13.39}& {	13.38}& {	13.24	}& {13.29	}& {13.42}& {	13.38}& {	13.36}& {	13.38}
\\
\cmidrule{1-12}
\multirow{2}{*}{Wine} & \multirow{2}{*}{Inhom. linear} & Accuracy & {$\underline{\mathbf{97.23 \pm 2.40}}$} & {$96.23 \pm 2.85$}	& {$96.82 \pm 2.05$} & {$\underline{96.68 \pm 2.44}$} & {$95.91 \pm 2.60$} & {$96.50 \pm 2.48$} & {$96.32 \pm 2.24$} & {$96.50 \pm 2.52$} & {$\underline{96.82 \pm 2.60}$}
\\
& & CPU time (s) & {102.90} & {102.28} & {98.78} & {101.24} & {98.27} & {96.16} & {100.98} & {104.46} & {99.43}
\\
\cmidrule{1-12}
\multirow{2}{*}{{Glass}} & \multirow{2}{*}{{Gaussian}} &{Accuracy} & {$\underline{\mathbf{38.92 \pm 4.39}}$} & {$35.14 \pm 6.29$} & {$35.85 \pm 5.43$} & {$\underline{37.74 \pm 3.77}$} & {$35.38 \pm 3.74$} & {$30.42 \pm 8.64$} & {$\underline{36.32 \pm 5.02}$} & {$35.14 \pm 4.39$} & {$29.48 \pm 9.02$}
\\
& & {CPU time (s)} & {2583.94} & {2803.87} & {2856.64} & {2753.53} & {2884.25} & {2869.89} & 2697.41 & {2705.17} & {2855.15}
\\
\cmidrule{1-12}
\multirow{2}{*}{Fuel} & \multirow{2}{*}{Gaussian} & Accuracy & {$\underline{\mathbf{56.82 \pm 6.64}}$} & {$55.10 \pm 5.26$} & {$49.03 \pm 5.11$} & {$\underline{57.80 \pm 5.10}$} & {$55.91 \pm 4.60$} & {$52.15 \pm 5.60$} & {$\underline{54.44 \pm 6.99}$} & {$54.03 \pm 4.48$} & {$53.23 \pm 6.43$}
\\
& & CPU time (s) & {5161.05} & {5102.44} & {4914.07} & {4832.09} & {4981.42} & {4875.20} & {5015.18} & {4988.35} & {5209.15}
\\
\cmidrule{1-12}
\multirow{2}{*}{{Car}} & \multirow{2}{*}{{Inhom. cubic}} & {Accuracy} & {$85.10 \pm 0.97$} & {$86.26 \pm 1.25$} & {$\underline{87.24 \pm 1.16}$} & {$85.24 \pm 1.85$} &  {$85.58 \pm 1.69$} & $\underline{\mathbf{87.56 \pm 2.02}}$ & {$84.98 \pm 1.78$} & {$\underline{86.21 \pm 1.57}$} & {$82.29 \pm 0.98$}
\\
& & {CPU time (s)} & {3060.21} & {3033.12} & {2984.90} & {2930.46} & {2941.68} & {2875.33} & {3026.86} & {3035.45} & {2981.57}
\\
\bottomrule
\end{tabular}
}
\caption{Robust model - nonlinear classifier - argmax decision function \eqref{dec_func_multiclass_nonlinear_argmax}.}
\label{tab_res_nonlinear_robust_maxdist}
    \end{subtable}

\caption{Detailed percentage results of average accuracy and standard deviation over 50 runs of the robust model \eqref{robust_nonlinear_SOCP} with nonlinear classifier. Classification is performed according to the argmin and argmax decision functions, respectively in Table \ref{tab_res_nonlinear_robust_mindist} and \ref{tab_res_nonlinear_robust_maxdist}. For each dataset, the kernel in the second column is chosen according to the corresponding best deterministic result of Tables \ref{tab_res_determ_mindist}-\ref{tab_res_determ_maxdist}. The best result for each $\ell_p$-norm is underlined. Overall, the best result is in bold.} \label{tab_res_nonlinear_robust}
\end{table}

The robust model generally yields an improved accuracy over the deterministic setting, particularly when using the argmax decision function (see Table \ref{tab_res_nonlinear_robust_maxdist}). For example, in the \emph{Wine} dataset, the robust model with an inhomogeneous linear kernel achieves an accuracy of $97.23\%$, compared to $96.41\%$ in the deterministic case. In general, the robust approach provides an accuracy that is either superior or comparable to the deterministic counterpart, especially when polynomial kernels are used, as observed in the \emph{Iris}, \emph{Wine}, and \emph{Car} datasets. Conversely, when using the Gaussian kernel (see the \emph{Glass} and \emph{Fuel} datasets), the performance tends to decline in the robust setting. This reduction is mainly due to the Gaussian kernel's sensitivity to data perturbations, which can negatively impact the decision boundary under strong uncertainty. It is worth noting that the best performances in the robust context are most frequently obtained using the $\ell_1$-norm and a small perturbation level ($\varepsilon = 10^{-3}$), suggesting a clear preference for these settings across the considered datasets.

In terms of CPU time, the robust approach with linear classifiers exhibits comparable computational times to the deterministic case. In contrast, when using kernel-based classifiers, the CPU time increases significantly, particularly for the Gaussian kernel (see the \emph{Glass} and \emph{Wine} datasets) and for datasets with a large number of observations (see the \emph{Car} dataset). This increase in computational time for larger datasets is consistent with the complexity analysis discussed at the end of Section \ref{sec_rob_nonlin_multiclass}.

Finally, in Table \ref{tab_comparison}, we compare the performance of our proposed method with the deterministic and robust SVM formulations presented in \cite{MagSpi2023}, considering both linear and nonlinear classifiers. We observe that in three out of five cases (see the \emph{Iris}, \emph{Wine}, and \emph{Glass} datasets) the robust formulations outperform their deterministic counterparts, emphasizing the benefit of incorporating uncertainty in SVM models. Furthermore, in three out of five datasets (see the \emph{Iris}, \emph{Wine}, and \emph{Fuel} datasets), our approach achieves the highest accuracy overall, demonstrating its competitiveness with state-of-the-art robust SVM alternatives.

\begin{table}[h!]
\centering
\resizebox{\textwidth}{!}{
\begin{tabular}{>{\color{black}}l| >{\color{black}}l >{\color{black}}l >{\color{black}}l >{\color{black}}l| >{\color{black}}l >{\color{black}}l >{\color{black}}l >{\color{black}}l}\toprule
Dataset & \multicolumn{4}{c|}{{Deterministic formulations}} & \multicolumn{4}{c}{{Robust formulations}}
\\
& \multicolumn{2}{c}{{Linear classifiers}} & \multicolumn{2}{c|}{{Nonlinear classifiers}} & \multicolumn{2}{c}{{Linear classifiers}} & \multicolumn{2}{c}{{Nonlinear classifiers}} \\
& Tables \ref{tab_res_determ_mindist}-\ref{tab_res_determ_maxdist} & \cite{MagSpi2023} & Tables  \ref{tab_res_determ_mindist}-\ref{tab_res_determ_maxdist}& \cite{MagSpi2023} & Tables \ref{tab_res_linear_robust_mindist}-\ref{tab_res_linear_robust_maxdist} & \cite{MagSpi2023} & Tables  \ref{tab_res_nonlinear_robust_mindist}-\ref{tab_res_nonlinear_robust_maxdist}& \cite{MagSpi2023}
\\
\toprule
Iris & $92.81$ & $88.92$ & $\underline{95.30}$ & $94.59$ & $95.35$ & $78.05$ & $\underline{\mathbf{95.46}}$ & $94.76$
\\
\cmidrule{1-9}
Wine & $97.00$ & $96.00$ & $\underline{97.41}$ & $97.14$ & $\underline{\mathbf{97.59}}$ & $96.32$ & $97.23$ & $96.09$
\\
\cmidrule{1-9}
Glass & $46.42$ & $46.98$ & $61.58$ & $\underline{65.13}$ & $47.85$ & $34.75$ & $39.86$ & $\underline{\mathbf{65.96}}$
\\
\cmidrule{1-9}
Fuel & $54.92$ & $56.62$ & $\underline{\mathbf{70.02}}$ & $58.30$ & $60.65$ & $\underline{69.16}$ & $62.99$ & $57.53$
\\
\cmidrule{1-9}
Car & $75.49$ & $75.42$ & $84.07$ & $\underline{\mathbf{96.00}}$ & $79.42$ & $69.99$ & $\underline{87.56}$ & 69.02
\\
\bottomrule
\end{tabular}
}
\caption{Percentage average accuracy comparison between deterministic and robust results obtained from the SVM formulations proposed in this work and in \cite{MagSpi2023}. For each formulation, the best result is underlined. Overall, the best result is in bold.}  \label{tab_comparison}
\end{table}

\section{Conclusions} \label{sec_conclusions}
In this paper, we have introduced novel Twin Parametric Margin Support Vector Machine (TPMSVM) models to address multiclass classification tasks under data uncertainty. From a methodological perspective, we have extended the TPMSVM techniques proposed in \cite{Peng2011} to handle multiclass settings, employing both linear and kernel-induced decision boundaries. As final decision functions, we have proposed two alternatives based on the argmin and argmax formulations.

Given the uncertain nature of real-world data, we have adopted a Robust Optimization approach by constructing a bounded-by-$\ell_p$-norm uncertainty set around each input data. This strategy aims to protect the deterministic models against uncertainties and prevents the worsening of the classification's performance. Then, we have derived robust counterparts of the deterministic models, and provided tractable reformulations in the form of Second Order Cone Programming models.

To validate the effectiveness of our proposal, we have evaluated the multiclass models on real-world datasets, considering different combinations of kernels and decision functions, while exploring various $\ell_p$-norms and levels of perturbations within the robust setting. The results of our experimental analysis show that robust solutions consistently yield higher accuracy compared to their deterministic counterparts, particularly when dealing with linear and polynomial kernels, but at the cost of increased computational time.

Regarding future research, various streams can originate from this work. First of all, extending the robust TPMSVM framework to account for uncertainties in the labels of input data could enhance the models' generalization capabilities. Secondly, different optimization under uncertainty methodologies could be explored to further robustify the TPMSVM approach. In particular, Chance-Constrained Programming and Distributionally Robust Optimization could represent effective strategies for enhancing robustness in multiclass classification under uncertainty. Thirdly, improving the interpretability and explainability of robust TPMSVM models, particularly when applied in sensitive domains such as medicine, merit additional investigation in future works. As robustness can make the models more complex and opaque, developing tools to analyze feature relevance under uncertainty and trace the effect of perturbations could offer valuable insights to the final users. Finally, from a computational perspective, two extensions are worth exploring. On the one hand, adopting alternative hyperparameter tuning techniques, such as Bayesian optimization, could lead to improved performance and computational efficiency. On the other hand, when working with large datasets, splitting the data into three subsets (training, tuning, and testing sets) rather than only two (training and testing sets), as suggested in \cite{BischlETAL2023}, could help mitigate overfitting and provide a more reliable assessment of model performance.

%
%
%
%
%

\section*{CRediT authorship contribution statement}
\textbf{Renato De Leone:} Conceptualization, Formal analysis, Funding acquisition, Investigation, Methodology, Resources, Supervision, Validation, Visualization, Writing - original draft, Writing - review \& editing. \textbf{Francesca Maggioni:} Conceptualization, Formal analysis, Funding acquisition, Investigation, Methodology, Project administration, Resources, Supervision, Validation, Visualization, Writing - original draft, Writing - review \& editing. \textbf{Andrea Spinelli:} Conceptualization, Data curation, Formal analysis, Investigation, Methodology, Software, Validation, Visualization, Writing - original draft, Writing - review \& editing.

\section*{Declaration of competing interest}
The authors declare that they have no known competing financial interests or personal relationships that could have appeared to influence the work reported in this paper.

\section*{Acknowledgements}
This work has been supported by ``ULTRA OPTYMAL - Urban Logistics and sustainable TRAnsportation: OPtimization under uncertainTY and MAchine Learning'', a PRIN2020 project funded by the Italian University and Research Ministry (grant number 20207C8T9M).

This study was also carried out within the MOST - Sustainable Mobility National Research Center and received funding from the European Union Next-GenerationEU (PIANO NAZIONALE DI RIPRESA E RESILIENZA (PNRR) - MISSIONE 4 COMPONENTE 2, INVESTIMENTO 1.4 - D.D. 1033 17/06/2022, CN00000023), Spoke 5 ``Light Vehicle and Active Mobility'' and by the PNRR MUR project ECS$\_$00000041-VITALITY- CUP J13C22000430001, Spoke 6. This manuscript reflects only the authors' views and opinions, neither the European Union nor the European Commission can be considered responsible for them.

\section*{References}

\bibliography{tpmsvm_bib}

\end{document}